\documentclass[10pt,journal]{IEEEtran}

\usepackage{mathrsfs}

\usepackage{latexsym}
\usepackage{amssymb}
\usepackage{graphicx}
\usepackage{amsmath}
\usepackage{epsfig}
\usepackage{latexsym}
\usepackage{indentfirst}
\usepackage{amsthm}
\usepackage{subfigure}
\usepackage{amsfonts}
\usepackage{picins}

\ifCLASSOPTIONcompsoc
\else
\fi
\ifCLASSINFOpdf
\else
\fi

\usepackage{graphicx}

\usepackage{multirow}
\newtheorem{theorem}{Theorem}
\newtheorem{assumption}{Assumption}

\newtheorem{lemma}{Lemma}

\newtheorem{proposition}{Proposition}

\newtheorem{definition}{Definition}

\graphicspath{{figures/}}

\begin{document}

\title{Realizing data features by deep nets}

\author{Zheng-Chu Guo,  Lei Shi, and Shao-Bo Lin
\IEEEcompsocitemizethanks{\IEEEcompsocthanksitem Z. C. Guo is with
 School of Mathematical Sciences, Zhejiang University,
Hangzhou, China. L. Shi is with
Shanghai Key Laboratory for Contemporary Applied Mathematics, School
of Mathematical Sciences, Fudan University, Shanghai, China. S. B. Lin is with   Department of
Mathematics, Wenzhou University, Wenzhou, China. Z. C. Guo and L. Shi are the co-first author. The corresponding author is S. B. Lin (email:
sblin1983@gmail.com).}}

  \IEEEcompsoctitleabstractindextext{

\begin{abstract}
This paper considers the power of deep neural networks (deep nets
for short) in realizing data features. Based on   refined covering
number estimates, we find that, to realize  some complex data
features, deep nets can improve the performances of shallow neural networks (shallow nets for short)
without requiring additional capacity costs. This verifies  the advantage of deep nets in realizing complex
features. On the other hand, to realize some simple data feature like the smoothness, we prove that, up to
a logarithmic factor, the approximation rate of deep nets is asymptotically identical to that of shallow nets,
provided that the depth is fixed. This exhibits a limitation of deep nets in realizing simple features.
\end{abstract}

}

\maketitle

\IEEEdisplaynotcompsoctitleabstractindextext

\IEEEpeerreviewmaketitle


\section{Introduction}
Deep learning  \cite{Hinton2006} is recognized to be a
state-of-the-art scheme in artificial intelligence and machine
learning and has recently triggered enormous research activities.
Deep neural networks (deep nets for short) is believed to be capable
of discovering deep features of data  which are important but are
impossible to be found by shallow neural networks (shallow nets for
short). It, however, simultaneously produces a series of challenges
such as the efficient computation, algorithmic solvability,
robustness, interpretability and so on. A direct consequence of
these challenges is that users hesitate to utilize deep learning in
learning tasks with high risk such as the clinical diagnosis and
financial investment, since it is not clear whether deep nets
perform essentially better than the scheme in hand. Thus, it is
urgent and crucial to provide the theoretical guidance on ``when do deep nets perform  better than shallow nets?''

Generally speaking, there are three steps to study the above problem. The first step is to correspond specific
real-world applications to some data features. For example, figures are assumed to be local similarity \cite{Wang2004};
earthquake forecasting is related to rotation-invariant features
\cite{Adeli2009}; and computer vision requires the spareness of
activated neurons on  the receptive field \cite{Wright2010}. The
second step is to connect these data features with a-priori
information which can be mathematically reflected by specific
properties of functions. In particular, local similarity usually corresponds to piece-wise smooth functions
\cite{Petersen2017}; rotation-invariance generally corresponds to
radial functions \cite{Chui2018a} and  sparseness on  the receptive
field frequently corresponds to sparseness in the spacial domain
\cite{Lin2018}. The last step is to pursue the outperformance of
deep nets   in approximating or learning   these application-related
functions. In fact, the outperformance of deep nets has been
rigorously  verified in approximating piece-wise smooth functions
\cite{Petersen2017}, rotation-invariant functions \cite{Chui2018a}
and sparse functions \cite{Lin2018}, which coincides with the
empirical evidences on image classification \cite{Krizhevsky2012},
earthquake prediction \cite{Vikraman2016} and computer vision
\cite{Lecun2015}.

With the rapid development in deep nets approximation theory, there
are numerous features that are proved to be realizable by deep nets
\cite{Chui2018,Lin2018,McCane2017,Mhaskar2016a,Petersen2017,Shaham2015}
with much less neurons than  shallow nets. Different from these
encouraging results, studies in learning theory showed that,
however, to realize these features, capacities of deep nets are much
larger than those of shallow nets with comparable number of free
parameters. In particular, under some specified capacity
measurements such as the number of linear regions
\cite{Montufar2013}, Betti numbers \cite{Bianchini2014}, number of
monomials \cite{Delalleau2011}, it was proved that the capacity of
deep nets increases exponentially with respect  to the depth but
polynomially with respect to the width. An extreme case is that
there exist  deep nets with two hidden layers whose capacity
measured by the pseudo-dimension is infinite
\cite{Maiorov1999b,Maiorov1999c}. The large capacity of deep nets
inevitably   makes the deep nets learner sensitive to noise and
requires a large amount of computations to find a good estimator.

In a nutshell, previous studies on advantages of deep nets showed
that deep nets are capable of   realizing  various
application-related data features, but it  requires additional
capacity costs. The first purpose of our study is to figure out
whether the large capacity of deep nets to realize  data features is
necessary. Our study is based on two interesting observations from
the literature
\cite{Chui1994,Lin2018,McCane2017,Mhaskar2016a,Petersen2017,Shaham2015,Yarotsky2017}.
One is that the number of layers of deep nets to realize various
data features is small, the order of which is at most the logarithm of the number
of free parameters. The other is that the magnitude of free
parameters is relatively small, which is at most a polynomial with
respect  to the number of free parameters. With these two findings,
we adopt the well known covering number \cite{Zhou2002,Zhou2003} to
measure the capacity of deep nets with controllable number of layers
and magnitude of weights and present a refined estimate of the
covering number of deep nets. In particular, we prove that  the
 covering number of deep nets with controllable depth and magnitude of weights  is similar as that of shallow nets with comparable
free parameters. This finding together with existing results in
approximation theory shows that, to realize  various features such
as sparseness, hierarchy, rotation-invariance and manifold
structures, deep nets improve the performance of  shallow
nets without  bringing additional capacity costs.

As is well known, advantages of deep nets in realizing  some special
features do not mean that deep nets are always better than shallow
nets.  Our second purpose is to demonstrate the necessity of deepening
networks in realizing some simple data features.
After building a close relation between approximation rates and
covering number estimates, we prove that if only the smoothness
feature is explored, then  up to a logarithmic factor,
approximation rates of shallow nets and deep nets with controllable
depth and magnitude of weights are asymptotically identical.
Combining the above two statements, we indeed present rigorous
theoretical verifications to support that deep nets are necessary in a large number of applications corresponding to
complex data features, in the sense that deep nets  realize  data
features without any additional capacity costs, but not all.

The rest of paper is organized as follows. In the next section,
after reviewing some advantages of deep nets in approximation, we
present a refined  covering number estimate  for deep nets.
 In Section \ref{Sec.Main results}, we give  a lower bound for deep nets approximation to show
the limitation for deep nets in realizing simple features. In the
last section, we draw a simple conclusion of this paper.

\section{Advantages of Deep Nets in Realizing Feature}\label{Sec.Advantage}
In this section, we study  advantages of deep nets in approximating
classes of functions with complex  features. After introducing some
mathematical concepts associated with deep nets, we review some
important results in approximation theory which show
that deep nets can realize some application-related features  that cannot be
approximated by shallow nets with comparable free parameters. Then,
we present a refined covering number estimate for deep nets to show
that deepening networks in some special way does not enlarge the
capacity of shallow nets.
\subsection{Deep nets with fixed structures}

Great progress of deep learning is built on deepening neural
networks with  structures. Deep nets with different structures have
been proved to be universal, i.e., \cite{Zhou2018,Zhou2018a} for
deep  convolutional nets, \cite{Kohler2017} for deep nets with tree
structures and \cite{Hanin2017} for deep fully-connected neural
networks.

Let $\mathbb{I}:=[-1,1]$ and
$x=(x^{(1)},\dots,x^{(d)})\in\mathbb I^d=[-1,1]^d$. Let
$L\in\mathbb N$ and $d_0,d_1,\dots,d_L\in\mathbb N$ with $d_0=d$.
Assume $\sigma_k:\mathbb R\rightarrow\mathbb R$, $k=1,\dots,L$, be
univariate nonlinear functions. For
$\vec{h}=(h^{(1)},\dots,h^{(d_k)})^T\in\mathbb R^{d_k}$, define
$\vec{\sigma}_k(\vec{h})=(\sigma_k(h^{(1)}),\dots,\sigma_{k}(h^{(d_k)}))^T$.
Deep nets with depth $L$ and width $d_j$ in the $j$-th hidden layer can be mathematically
represented as
\begin{equation}\label{Def:DFCN}
     h_{\{d_0,\dots,d_L,\sigma\}}(x)=\vec{a}\cdot
     \vec{h}_L(x),
\end{equation}
where
\begin{equation}\label{Def:layer vector}
    \vec{h}_k(x)=\vec{\sigma}_k(W_k\cdot
    \vec{h}_{k-1}(x)+\vec{b}_k),\qquad k=1,2,\dots,L,
\end{equation}
$\vec{h}_{0}(x)=x,$ $\vec{a}\in\mathbb R^{d_L}$,
$\vec{b}_k\in\mathbb R^{d_k},$
 and $W_k=(W_k^{i,j})_{i=1,j=1}^{d_{k},d_{k-1}}$
be a $d_{k}\times
 d_{k-1}$ matrix. Denote by $\mathcal H_{\{d_0,\dots,d_L,\sigma\}}$ the set of
all these deep nets. When $L=1$,  the function defined by
(\ref{Def:DFCN}) is the classical shallow net.

The structure of  deep nets can be  reflected by structures of the
weight matrices $W_k$ and parameter vectors $\vec{b}_k$ and
$\vec{a}$, $k=1,2,\dots,L$. For examples, deep convolutional neural
networks corresponds to  Toeplitz-type weight matrices
\cite{Zhou2018a} and deep nets with tree structures usually
correspond extremely sparse weight matrices \cite{Mhaskar2016a}.
Throughout this paper, a deep net with specific structures
refers to a deep nets with specific structures of all
$W_k,\vec{b}_k$, $k=1,\dots,L$ and $\vec{a}$. Figure 1 shows two
structures for deep nets.
\begin{figure}[!t]
\begin{minipage}[b]{.49\linewidth}
\centering
\includegraphics*[scale=0.17]{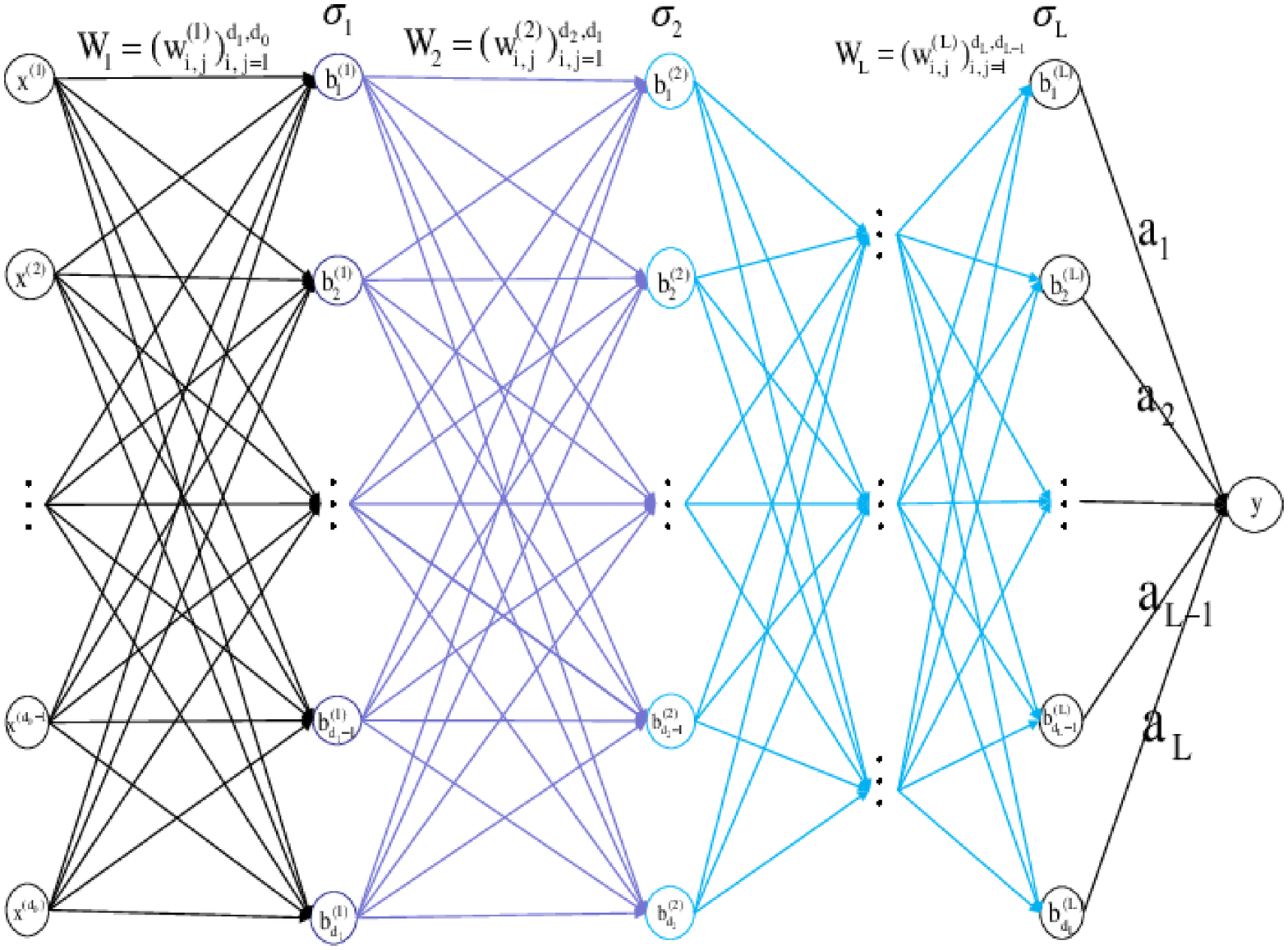}
\centerline{{\small (a) Deep fully-connected nets}}
\end{minipage}
\hfill
\begin{minipage}[b]{.49\linewidth}
\centering
\includegraphics*[scale=0.17]{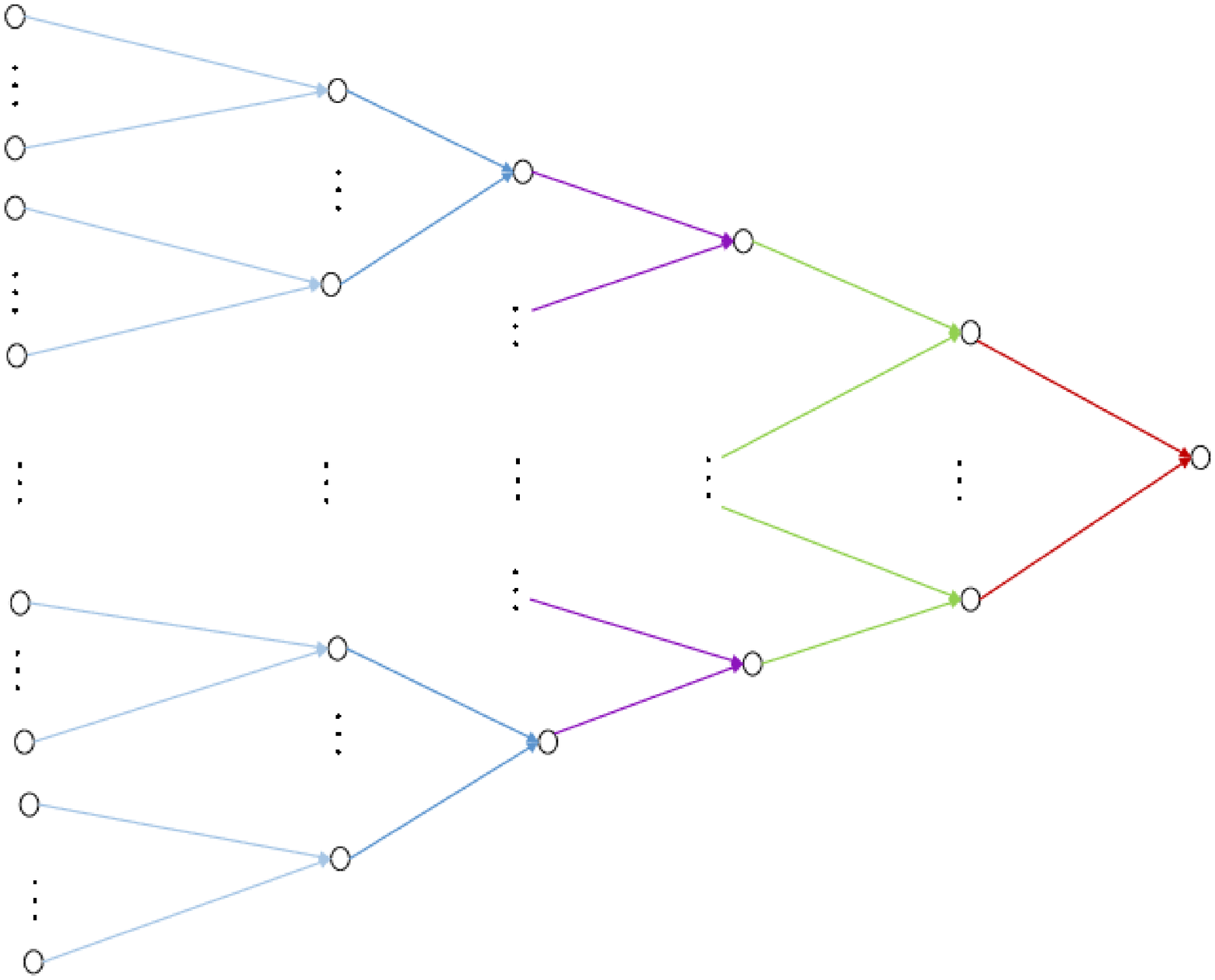}
\centerline{{\small (b) Deep nets with tree structure}}
\end{minipage}
\hfill \caption{ Structures for deep nets }
\end{figure}

Although deep fully-connected neural networks possess better
approximation ability than other networks, the number of free
parameters of this type networks  is
\begin{equation}\label{Number of parameters}
      \mathcal A_L=d_L+\sum_{k=1}^L (d_{k-1}d_k+d_{k}),
\end{equation}
which is huge when the width and depth are  large. A recent focus in
deep nets  approximation is to pursue the approximation ability of
deep nets with fixed structures. Up till now,  numerous theoretical
results \cite{Poggio2017,Zhou2018a,Chui2018a,Petersen2017} showed
that the approximation ability of deep fully-connected neural
networks can be maintained by deep nets with some special structures
with much less free parameters.

In this paper, we are   interested  in deep nets with structures.
For $k=1,\dots,L$, we assume that the structure of deep nets is
fixed and  there are $\mathcal F_{k,w}$ free parameters in $W_k$,
$\mathcal F_{k,b}$ free thresholds in $\vec{b}_k$ and $\mathcal
F_{L,a}$ free parameters in $\vec{a}$. Then, there are totally
\begin{equation}\label{papr for spasre}
     n:=\sum_{k=1}^L(\mathcal F_{k,w}+\mathcal
     F_{k,b})+\mathcal F_{L,a}
\end{equation}
free parameters in the deep nets. We assume further $n \ll
\mathcal A_L$. Throughout the paper, we say there are $\mathcal F_{k,w}$ free parameters in $W_k$,
if the weight matrix $W_k$ is generated through the following three ways. The first way is that the matrix
has $\mathcal F_{k,w}$ entries that can be determined freely, while the
reminder $d_{k}d_{k-1}-\mathcal F_{k,w}$ entries are fixed, e.g., the weight matrix in deep nets with tree structures. The second way
 is that the weight matrix $W_k$ is exactly generated by $\mathcal F_{k,w}$ free parameters, e.g., the Toeplitz-type
weight matrix in deep convolutional neural networks. The third way is that the weight matrix is generated jointly by both
way above, that is, part of the weight matrix is fixed, while the remaining part
are totally generated by $\mathcal F_{k,w}$ free parameters. Denote by $\mathcal H_{\{n,L,\sigma\}}$ the set of
all these deep nets with $L$ hidden layers, fixed structure and $n$
free parameters. Denote further
\begin{eqnarray}\label{Hypothesis space}
       &&\mathcal H_{\{n,L,\sigma,\mathcal R\}}=
       \{h_{n,L,\sigma}
      \in \mathcal H_{\{n,L,\sigma\}}: \nonumber\\
      &&|w_{k}^{i,j}|,|b_k^{i}|,|a_i|\leq \mathcal R,
      1\leq i\leq d_{k},1\leq j\leq d_{k-1},\nonumber\\
      && 1\leq k\leq
      L\}
\end{eqnarray}
the set of deep nets whose weights and thresholds are uniformly
bounded by $\mathcal R$, where $\mathcal R$ is some positive number
which may depend on $n$, $d_k$, $k=0,1\dots,L$ and $L$. We  aim at
studying the approximation ability and capacity   of $\mathcal
H_{\{n,L,\sigma,\mathcal R\}}$.

It should be mentioned that the boundedness assumption in
(\ref{Hypothesis space}) is necessary. In fact, without such an
assumption, \cite{Maiorov1999b,Ismailov2014} proved that for
arbitrary $\varepsilon>0$ and arbitrary continuous function $f$,
a deep net with two hidden layers and finitely many
free parameters is fully able to generate an approximation $H_{f}$, such that
\begin{equation}\label{Universal approximation}
          \|f-H_f\|_{L_p(\mathbb I^d)}\leq \varepsilon.
\end{equation}
This implies that the capacity of deep nets with two hidden layers
and finitely many free parameters is comparable with that of
$L_p(\mathbb I^d)$, showing its extremely large capacity. Therefore, to further control the
capacity of deep nets, the boundedness assumption  has been employed in large literature
\cite{Kohler2017,Lin2018,Petersen2017}.

\subsection{A fast review for   realizing   data features by deep nets}

In approximation and learning theory, data features are usually
formulated by a-priori information for corresponding functions, like
the target function \cite{Chui1994} for approximation,  regression
function \cite{Lin2018} for regression and  Bayes decision function
\cite{Lin2017}
 for classification.
Studying advantages of deep nets in approximating functions with
different a-priori information  is a classical topic. It can date
back to 1994, when \cite{Chui1994} deduced the localized
approximation property of deep nets which is far beyond the
capability of shallow nets.

The localized approximation of a neural network  shows that if the
target function is modified only on a small subset of the Euclidean
space, then only a few neurons, rather than the entire network, need
to be retrained. We refer to \cite[Def.2.1]{Chui1994} for a formal
definition of localized approximation. Since the localized
approximation is an important step-stone in approximating
piecewise smooth functions \cite{Petersen2017} and sparse functions
in spacial domains \cite{Lin2018}, deep nets perform much better
than shallow nets in related applications such as image processing
and computer vision \cite{Goodfellow2016}.
 The following proposition,
which can be found in \cite[Theorem 2.3]{Chui1994} (see also
\cite{Lin2018}), shows the localized approximation property of deep
nets.

\begin{proposition}\label{Proposition}
Suppose that $\sigma:\mathbb R\rightarrow\mathbb R$ is a bounded
measurable function with the sigmoidal property
\begin{equation}\label{sigmoidal}
        \lim_{t\rightarrow-\infty}\sigma(t)=0,\qquad \lim_{t\rightarrow+\infty}
        \sigma(t)=1.
\end{equation}
Then  there exists a deep net with two hidden layers, $2d+1$ neurons
and activation function $\sigma$ provides localized approximation.
\end{proposition}

Rotation-invariance,  is another popular data feature, which abounds
in statistical physics \cite{LinH2017}, earthquake early warning
\cite{Satriano2011}  and image rendering \cite{Meylan2006}.
Mathematically, rotation-invariant property corresponds to a radial
function which is by definition a function whose value at each point
depends only on the distance between that point and the origin. In
the nice papers  \cite{Konovalov2008,Konovalov2009}, shallow nets
were proved to  be incapable of embodying  rotation-invariance
features. To show the power of depth in approximating radial
functions, we present the definition of smooth radial function as
follows.
\begin{definition}\label{Definition:smoothness}
Let $\mathbb A\subset\mathbb R$, $c_0>0$ and $r=s+v$ with
$s\in\mathbb N_0:=\{0\}\cup\mathbb N$ and $0<v\leq 1$. We say a
univariate function $g:\mathbb A\rightarrow\mathbb R$ is
$(r,c_0)$-Lipschitz continuous if $g$ is $s$-times differentiable
and its $s$-th derivative satisfies the Lipschitz condition
\begin{equation}\label{lip}
          |g^{(s)}(t)-g^{(s)}(t_0)|\leq c_0|t-t_0|^v,\qquad\forall\
          t,t_0\in\mathbb A.
\end{equation}
Denote by $Lip^{(r,c_0)}_{\mathbb A}$ the set of all
$(r,c_0)$-Lipschitz continuous functions defined on $\mathbb A$.
Denote also by $Lip^{(\diamond,r,c_0)}$ the set of radial functions
$f=g(\|x\|_2^2)$ with $g\in Lip^{(r,c_0)}_{[0,1]}$.
\end{definition}
The following proposition, which can be found in \cite{Chui2018a},
shows that deep nets can realize   rotation-invariance  and
smoothness features of  target functions, simultaneously.

\begin{proposition}\label{Proposition:almost optimal}
Let $d\geq 2$ and $1\leq p\leq \infty$. If $\sigma$ is the logistic
function, i.e. $\sigma(t)=\frac{1}{1+e^{-t}}$, then for arbitrary
$f\in Lip^{(\diamond,r,c_0)}$, there is an $h\in \mathcal
H_{\{n,3,\sigma, \mathcal R\}}$   such that
\begin{equation}\label{Optimal approx}
   \|f-h\|_{L^p(\mathbb I^d)}\leq C_1n^{-r}.
\end{equation}
Furthermore,  for arbitrary $h'\in \mathcal H_{\{n,1,\sigma,\mathcal
R\}}$,  there always exists a function $f_0\in
Lip^{(\diamond,r,c_0)}$ satisfying
\begin{equation}\label{limitation for shallow}
   \|f_0-h'\|_{L^\infty(\mathbb I^d)}
   \geq
    C_2 n^{-r/(d-1)},
\end{equation}
where  $C_1$, $C_2$ are constants independent of
$d_0,d_1,\dots,d_L$ or $n$.
\end{proposition}

Numerous  learning problems \cite{LinT2008} in computer vision, gene
analysis and speech processing involve high dimensional data. These
data are often governed by many fewer variables, producing
manifold-structure features in a high dimensional ambient space. A
large number of theoretical studies
\cite{Chui2018,Shaham2015,Ye2008} have
 revealed that shallow nets are difficult to realize   smooth
 and
manifold-structure features simultaneously. Conversely, deep nets,
as studied in \cite{Shaham2015,Chui2018},  is capable of reflecting
these features, which is shown by the following proposition \cite{Shaham2015} (see also \cite{Chui2018}).

\begin{proposition}\label{Proposition:almost optimal}
Let $\Gamma\subset\mathbb I^d$ be a smooth $d'$-dimensional compact
manifold (without boundary) with $d'\ll d$. If $\sigma$ is the ReLU
activation function, i.e. $\sigma(t)=\max\{t,0\}$, and $f$ is
defined on $\Gamma$ and twice differentiable, then there exists a
$g\in \mathcal H_{ n,4,\sigma, \mathcal R }$   such that
\begin{equation}\label{Optimal approx}
   \|f-g\|_{L^2(\mathbb I^d)}\leq C_3n^{-\frac{2}{d'}}.
\end{equation}
where  $C_3$ is a   constant  independent of $d_0,d_1,\dots,d_L$ or
$n$.
\end{proposition}

The previous studies showed that, compared with shallow nets, deep nets
equipped with fewer parameters are enough to approximate functions
with complex features to the same accuracy. In the following
Table I, we list some literature on studying the advantages of
realizing data futures.
\begin{table}[h]
\begin{center}
\begin{tabular}{|l|l|l|l|l|}
\hline References & Features & $\sigma $& $L$\\
\hline \cite{Chui1994,Chui2018} & Localized approximation  & Sigmoidal & 2\\
\hline \cite{Lin2018} & Sparse+Smooth & Sigmoidal& 2  \\
\hline \cite{Shaham2015} & Smooth+Manifold & ReLU & 4 \\
\hline \cite{Mhaskar2016,Kohler2017} &Hierarchical+Smooth  &Sigmoidal & Hierarchical \\
\hline \cite{Petersen2017} &piecewise smooth & ReLU & Finite \\
\hline \cite{Safran2016} & $\ell_1$ radial+smooth & ReLU & $\log(\varepsilon^{-1})$   \\
\hline \cite{LinH2017,Rolnick2017} & Sparse (frequency) & Analytic & $\log(\varepsilon^{-1})$   \\
\hline
\end{tabular}%
\end{center}
\caption{Powers of deep nets in approximation (within accuracy
$\varepsilon$)}
\end{table}

\subsection{Covering number estimates}

In the above subsection, we have reviewed some results on the
advantages of deep nets in realizing data features. However, it does
not mean that deep nets are better than shallow nets, since we do
not know what price is paid for such advantages in approximation.
In this subsection,  we use the covering number, which is widely
used in learning theory
\cite{Lin2017,Shi2011,Shi2013,Zhou2002,Zhou2003}, to measure the
capacity of $\mathcal H_{n,L,\sigma, \mathcal R }$ and then unify
the comparison within the same framework to show the outperformance
of deep nets.

 Let $\mathbb B$ be a Banach
space and $V$ be  a subset of $\mathbb B$. Denote by $\mathcal
N(\varepsilon,V,\mathbb B)$ the $\varepsilon$-covering number
 of $V$ under the metric of $\mathbb B$, which is
the minimal number of elements in an $\varepsilon$-net of $V$. If
$\mathbb B=L_1(\mathbb I^d)$, we denote $\mathcal
N(\varepsilon,V):=\mathcal N(\varepsilon,V,L_1(\mathbb I^d))$ for
brevity. Our purpose is a tight   bound  for  covering numbers of
$\mathcal H_{n,L,\sigma,\ \mathcal R }$. To this end, we need the
following assumption.
\begin{assumption}\label{Assumption:activation}
 For arbitrary $t\in\mathbb R$ and every $k\in\{1,\dots,L\}$, assume
\begin{equation}\label{Lip for phi}
     |\sigma_k(t)-\sigma_k(t')|\leq c_1|t-t'|
\end{equation}
and
\begin{equation}\label{bound for phi}
      |\sigma_k(t)|\leq c(|t|+1)
\end{equation}
for some $c,c_1\geq1$.
\end{assumption}

To be detailed,  (\ref{Lip for phi}) shows the Lipchitz continuous
property of $\sigma_k$ and (\ref{bound for phi}) exhibits the
linear increasing condition of $\sigma_k$. These assumptions have been utilized in
\cite{Makovoz1996,Kurkova2007,Lin2018} to quantify  covering
numbers of neural networks with different structures. We can see
that almost all widely used activation functions such as the
logistic function, hyperbolic tangent sigmoidal function
$
         \sigma(t)=\frac12(\tanh(t)+1)
$ with $\tanh(t)=(e^{2t}-1)/(e^{2t}+1)$,   arctan sigmoidal function
$
         \sigma(t)=\frac1{\pi}\arctan(t)+\frac12,
$
 Gompertz function
$
        \sigma(t)=e^{-a e^{-bt}}
$ with $a,b>0$, ReLU $\sigma(t)=\max\{t,0\}$,  and Gaussian function
$\sigma(t)=e^{-t^2}$ satisfy  Assumption
\ref{Assumption:activation}. With this assumption, we present our
first main result in the following theorem, whose proof can be found
in Appendix A.

\begin{theorem}\label{Theorem:covering number}
Let $\mathcal H_{n,L,\sigma, \mathcal R}$ be defined by
(\ref{Hypothesis space}). Under Assumption
\ref{Assumption:activation}, there holds
\begin{eqnarray*}
  \mathcal N\left( \varepsilon,\mathcal H_{n,L,\sigma,\mathcal
       R\}}\right)
     \leq
   \left(c_2\mathcal RD_{\max}\right)^{3(L+1)^2n}\varepsilon^{-
   n},
\end{eqnarray*}
where $D_{\max}:=\max_{0\leq\ell\leq L}d_\ell$ and $c_2\geq 1$ is a
constant depending only on $c,c_1$ and $d$.
\end{theorem}
For $\sigma$ satisfying Assumption \ref{Assumption:activation},  it
was deduced in \cite{Maiorov2006,Gyorfi2002} that
\begin{equation}\label{covering for shallow}
        \log\mathcal N(\varepsilon,\mathcal H_{n,1,\sigma, \mathcal
     R})=\mathcal O\left(n
        \log\frac{C_4\mathcal R}{\varepsilon}\right),
\end{equation}
where $C_4$ is a constant independent of $\varepsilon$  or $n$. From
Theorem \ref{Theorem:covering number}, we can derive
\begin{equation}\label{covering for DFCN}
        \log\mathcal N(\varepsilon,\mathcal H_{n,L,\sigma, \mathcal
     R})=\mathcal O\left(L^2n
        \log\frac{C_5\mathcal R}{\varepsilon}\right)
\end{equation}
for some $C_5$ independent of $\varepsilon$, $L$,
$d_0,d_1,\dots,d_L$ or $n$. Comparing (\ref{covering for DFCN}) with
(\ref{covering for shallow}), we find that, up to a logarithmic
factor, deep nets do not essentially enlarge the capacity of shallow
nets, provided that they have same number of free parameters and the
depth of deep nets is at most $\log n$. Noting that the depths of
deep nets in Table I all satisfy this constraint, Theorem
\ref{Theorem:covering number} shows that to realize various data
features presented in Table I, deep nets can improve the performance
of shallow nets without imposing additional capacity costs. Therefore,
 Theorem \ref{Theorem:covering number} together with Table I yields
the reason why deep nets perform much better than shallow nets in
some complex learning tasks such as image processing and computer
vision.

Recently, \cite{Harvey2017} presented a tight VC-dimension bounds
for piecewise linear neural networks.  In particular, they proved
that
\begin{equation}\label{VCdimension for deep}
        VCDim(sgn(\mathcal H_{\{d_0,\dots,d_L,\sigma\}}))
        =
        \mathcal O(L n\log n),
\end{equation}
where $VCDim(V)$ denotes the VC-dimension of the set $V$ and
$sgn(V):=\{x\rightarrow sgn(f(x)):f\in V\}$, where $sgn(f(x))=1$ if $f(x)\geq 0$ and $sgn(f(x))=-1$ otherwise. Using the standard
approach in \cite[Chap.9]{Gyorfi2002}, we can derive
\begin{equation}\label{covering for DFCN for linear}
        \log\mathcal N(\varepsilon,\mathcal H_{n,L,\sigma, \mathcal
     R })=\mathcal O\left(Ln
        \log\frac{C_5\mathcal R}{\varepsilon}\right)
\end{equation}
provided that $\sigma_1=\sigma_2=\cdots=\sigma_L$ are piecewise linear,
where $C_6$ is a constant independent of
$\varepsilon,L,d_0,\dots,d_L$ or $n$. Comparing (\ref{covering for
DFCN for linear}), there is an additional $L$ in our analysis. The
reason is that we focus on all activation functions satisfying
(\ref{Assumption:activation}) rather than piecewise activation
functions. It should be also mentioned that similar covering number
estimates for deep nets with tree structures has been studied in
\cite{Chui2018a,Kohler2017,Lin2018}. We highlight that   different
structures yield essentially non-trivial approaches. In fact, due to
tree structures, the approach in \cite{Chui2018a,Kohler2017,Lin2018}
is just to decouple layers by using the boundedness and Liptchiz
property of activation functions. However, in estimating covering
number of deep nets with arbitrarily fixed structure, we need a
novel matrix-vector transformation technique, as presented in
Appendix A.

\section{Necessity of the Depth}\label{Sec.Main results}

Previous studies showed that, to realize some complex data features,
deep nets can improve the performance of shallow nets without
additional capacity costs. In this section, we study in a
different direction to prove that, to realize some simple data
features, deep nets are  not essentially better than shallow nets.

\subsection{Limitations of deep nets approximation}

Smoothness or regularity is a widely used feature that has been
adopted in a vast literature
\cite{Chui1994,Chui1996,Konovalov2008,Konovalov2009,Maiorov1999b,Maiorov1999c,Yarotsky2017}.
To present the approximation result, we at first introduce the
following definition.
\begin{definition}\label{definition:lip for d}
Let $c_0>0$ and $r=s+v$ with $s\in\mathbb N_0:=\{0\}\cup\mathbb N$
and $0<v\leq 1$. We say a   function $f:\mathbb
I^d\rightarrow\mathbb R$ is $(r,c_0)$-smooth if $f$ is $s$-times
differentiable and for every $\alpha_j\in \mathbb N_0$,
$j=1,\dots,d$ with $\alpha_1+\dots+\alpha_d=s$, its $s$-th partial
derivative satisfies the Lipschitz condition
\begin{equation}\label{lip}
          \left|\frac{\partial^sf}{\partial x_1^{\alpha_1}\dots\partial
          x_d^{\alpha_d}}
          (x)-\frac{\partial^sf}{\partial x_1^{\alpha_1}\dots\partial
          x_d^{\alpha_d}}
          (x')\right|\leq c_0\|x-x'\|^v,
\end{equation}
where $x,x'\in\mathbb I^d$ and $\|x\|$ denotes the Euclidean norm of
  $x\in\mathbb R^d$.  Denote by $Lip^{(r,c_0)}$ the set of all
$(r,c_0)$-smooth functions defined on $\mathbb I^d$.
\end{definition}

Approximating smooth functions  is a classical  topic in neural
networks approximation. It is well known   that   the approximation
rate can be as fast as $\mathcal O(n^{-r/d})$ for neural networks
with $n$ free parameters. In particular, the Jackson-type error
estimate
\begin{equation}\label{upper bound shallow}
    \mbox{dist}(Lip^{(r,c_0)},\mathcal H_{\{n,1,\sigma,\mathcal R},L_p(\mathbb I^d))\leq
   C_1'n^{-\frac{r}{d}}
\end{equation}
has been established \cite{Maiorov2000} for shallow nets with
analytic activation functions, where
\begin{eqnarray*}
     \mbox{dist}(U,V, L_p(\mathbb I^d))
     &:=&\sup_{f\in U} \mbox{dist}(f,V, L_p(\mathbb I^d))\\
     &:=&\sup_{f\in U} \inf_{g\in V}\|f-g\|_{L_p(\mathbb
       I^d)}
\end{eqnarray*}
denotes the deviations of $U$ from $V$ in  $L_p(\mathbb I^d)$ for
$U,V\subseteq L_p(\mathbb I^d)$.
 Similar results has been derived in
\cite{Chui1994} with deep nets with two hidden layers and a sigmoidal
activation function. Recently, \cite{Yarotsky2017} derived an error
estimate taking the form of
\begin{equation}\label{upper bound relu}
    \mbox{dist}(Lip^{(r,c_0)},\mathcal H_{n,L,\sigma ,\mathcal R },L_p(\mathbb I^d))\leq
   C_2'n^{-\frac{r}{d}}\log n
\end{equation}
for deep nets with $L=\log n$ and  ReLU activation functions. We would like to point out that,
for shallow nets with ReLU activation functions, estimates (\ref{upper bound relu}) holds only for $0<r\leq 1$,
which is also considered as the approximation bottleneck of shallow nets. The paper
\cite{Yarotsky2017} showed that deepening the networks can
overcome this bottleneck for shallow nets. However, it should be mentioned from
  (\ref{upper bound shallow}) that for other activation functions except the ReLU activation functions,
such a bottleneck does not exist. Thus, the paper \cite{Yarotsky2017} indeed
conduct a nice analysis on the necessity of deepening ReLU nets. However, their established results can not illustrate the necessity of depth.

In the following theorem that will be proved in Appendix C, we show that
deep nets cannot be essentially better than shallow nets in realizing the
smoothness feature.

\begin{theorem}\label{Theorem:lower bound for deep nets}
Let $1\leq p\leq \infty$, $L\in \mathbb N$. Then
\begin{eqnarray}\label{lower bound deep nets in theorem}
  &&\mbox{dist}(Lip^{(r,c_0)},\mathcal
     H_{n,L,\sigma,\mathcal
     R},L_1(\mathbb I^d))   \nonumber \\
   &\geq&
   C[L^2n\log_2n \log_2(\mathcal R
       D_{\max})]^{-\frac{r}d},
\end{eqnarray}
where $C$ is a constant depending only on $c,c_0,c_1,d$ and $r$.
\end{theorem}

Combining the estimates (\ref{lower bound deep nets
in theorem}) and (\ref{upper bound shallow}),  and noting 
$$
     \|f\|_{L^1(\mathbb I^d)}\leq C_{d,p}\|f\|_{L^p(\mathbb I^d)} 
$$
with $C_{d,p}$ a constant depending only on $d$ and $p$, 
we see that, when $L$ is not
too large, deep nets cannot essentially improve the approximation rate if one
only considers the smoothness feature. When $L$ is too large, it follows from Theorem
\ref{Theorem:covering number} that we will need additional capacity cost for deep nets
to improve the approximation ability of shallow nets. In other
words, the smoothness feature is not sufficient to judge whether the
depth of neural networks is necessary.

\subsection{Remarks and discussions}

Limitations of the approximation capabilities of shallow nets
  were firstly studied in \cite{Chui1996} in terms of
providing lower bounds of approximation of smooth functions in the
minimax sense. Recently, \cite{Lin2017a} highlighted that there
exists a probabilistic measure, under which, all smooth functions
cannot be approximated by shallow nets very well with  high
confidence.  In another two interesting papers
\cite{Kurkova2017,Kurkova2018}, limitations of shallow nets were
presented in terms of establishing lower bound of approximating
functions with some variation restrictions. However, due to these results,
it is still not clear whether the depth of neural networks is necessary, if
only the smoothness information is given.

Theorem \ref{Theorem:lower bound for deep nets} goes further along this
direction and presents a negative answer. In Theorem
\ref{Theorem:lower bound for deep nets}, to realize smoothness
features, deep nets perform almost the same as shallow nets. This
result verifies the common consensus  that deep learning outperforms
shallow learning in some ``difficult'' learning tasks
\cite{Goodfellow2016}, but not always. Moreover, our
result also implies that whether deep nets can help to improve the performance
of the existing learning schemes depends on what features for data we are exploring.
Combing our work with
\cite{Mhaskar2016,Mhaskar2016a,Kohler2017,Petersen2017,McCane2017,Chui2018,Chui2018a,Lin2018},
we can illustrate the comparison between shallow and deep nets in
Figure 2.
\begin{figure}
\begin{minipage}[b]{.45\linewidth}
\centering
\includegraphics*[scale=0.15]{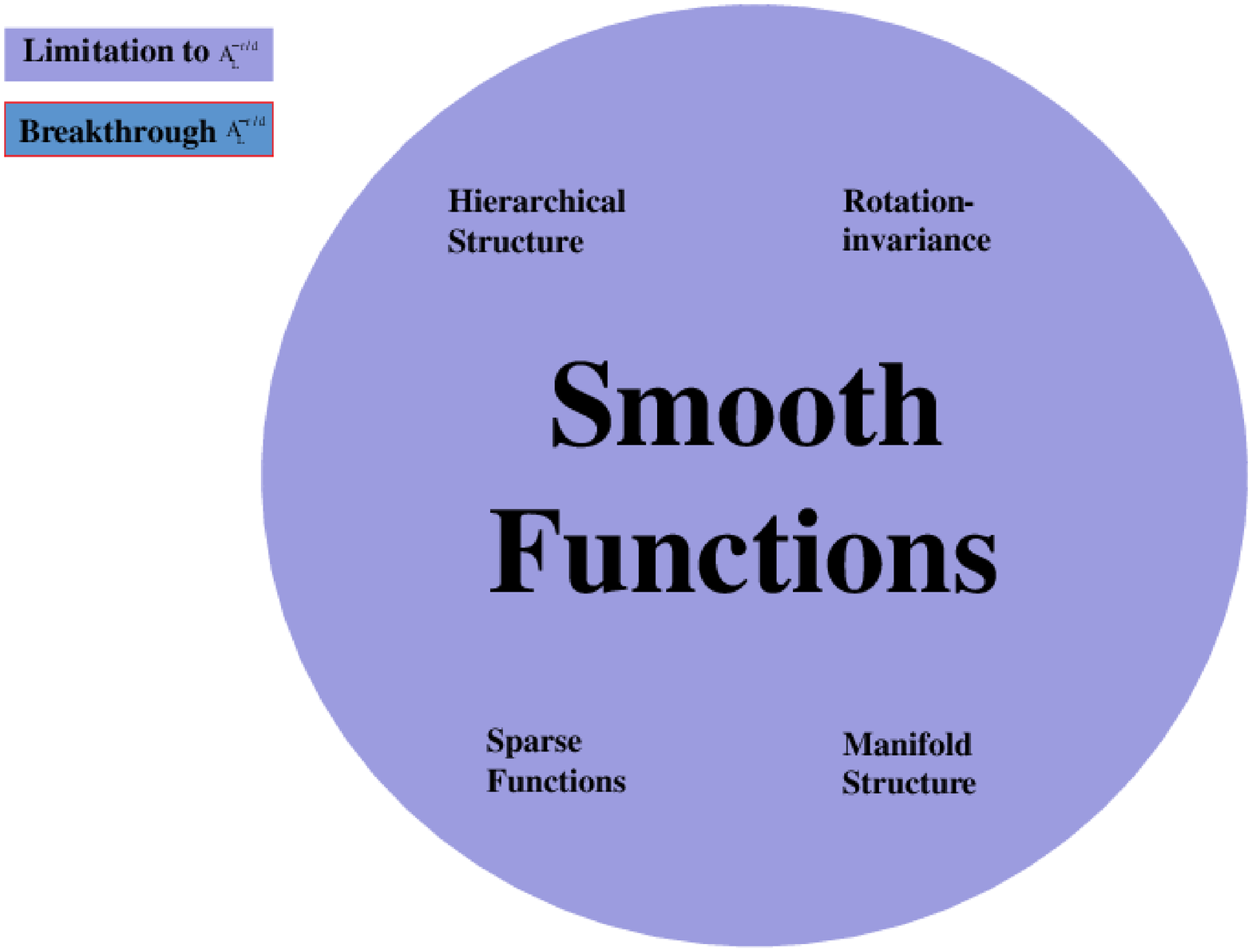}
\centerline{{\small (a) Approximation by shallow nets}}
\end{minipage}
\hfill
\begin{minipage}[b]{.45\linewidth}
\centering
\includegraphics*[scale=0.15]{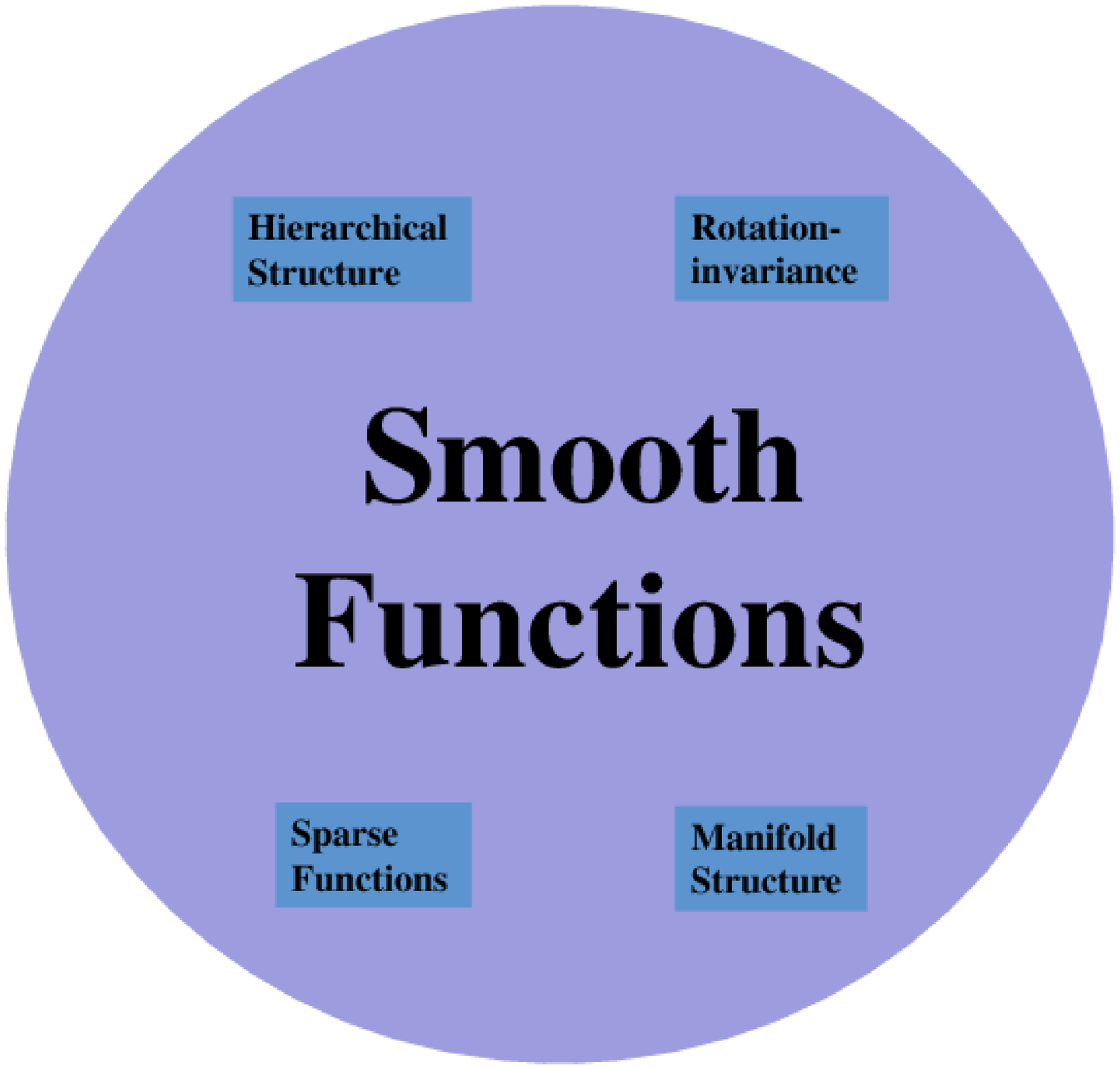}
\centerline{{\small (b) Approximation by DFCNs}}
\end{minipage}
\hfill \caption{Comparison between deep and shallow nets
}\label{Fig:com}
\end{figure}

We declare that Theorem \ref{Theorem:lower bound for deep nets} only
presents limitations of  deep nets in realizing smooth features. As
shown in Figure 2,  if more features are  explored, we believe that
the approximation rate of deep nets can break through the lower bound
presented in (\ref{lower bound deep nets in theorem}).

\section{Conclusion}

In this paper, we study the advantages and limitations of deep nets
in realizing different data features. Our results showed that, in
realizing some complex data features such as the
rotation-invariance, manifold structure, hierarchical structure,
sparseness, deep nets can improve the performance of shallow nets without additional capacity costs. We also exhibit that for
some simple data features like the smoothness, deep nets performs
essentially similar as shallow nets.

\section*{Appendix A: Proof of Theorem \ref{Theorem:covering
number}}

For $1\leq \ell \leq L$, let  $\mathcal W^*_{\mathcal F_{\ell,w}}$
be the set of $d_\ell\times d_{\ell-1}$ matrices with fixed
structures and total $\mathcal F_{\ell,w}$ free parameters and
$\vec{B}^*_{\mathcal F_{\ell,b}}$ be the set of $\mathcal
F_{\ell,b}$-dimensional vectors with fixed structures and total
$\mathcal F_{\ell,b}$ free parameters. Denote
$$
      \mathcal W_{\mathcal
F_{\ell,w}}:=\{W\in \mathcal W^*_{\mathcal F_{\ell,w}}:
|W^{i,j}|\leq \mathcal R\}
$$
and
 $$
 \vec{\mathcal B}_{\mathcal F_{\ell,b}}:=\{\vec{b}\in
\vec{B}^*_{\mathcal F_{\ell,b}}: |b_{i}|\leq \mathcal R\}.
$$
For $x\in\mathbb I^d$, let  $\vec{\mathcal H}_{0}=\{x\}$ and define
iteratively
\begin{eqnarray}
       &&\vec{\mathcal H}_{\ell}
        =
      \{\vec{h}_{\ell}(x)=\vec{\sigma_\ell}(W_{\ell}\vec{h}_{\ell-1}(x)+\vec{b}_\ell):\\
      &&\vec{h}_{\ell-1}\in\vec{\mathcal
     H}_{\ell-1}, W_{\ell}\in \mathcal W_{\mathcal F_{\ell,w}}, \vec{b}_\ell\in\vec{\mathcal B}_{\mathcal
     F_{\ell,b}}\},\quad \ell=1,2,\dots,L. \nonumber
\end{eqnarray}
For each
$\vec{h}_\ell=(h_\ell^1,\dots,h_\ell^{d_\ell})\in\vec{\mathcal
H}_\ell$, define $\|\vec{h}_\ell\|_{*,d_\ell}:=\max_{1\leq i\leq
d_\ell}\|h_\ell^i\|_{L^1({\mathbb I}^d)}$.
  The following lemma  devotes to the uniform bound
of functional vectors in $\vec{\mathcal H}_{\ell}$. In our analysis, we always assume that the activation functions satisfy
Assumption \ref{Assumption:activation} with uniform constants $c$ and $c_1$. Moreover, we also suppose that $\mathcal R \geq 1$ and
$c\geq 1$.

\begin{lemma}\label{Lemma:bounds}
For each $\ell=1,2,\dots,L$  and $\vec{h}_\ell\in\vec{\mathcal
H}_{\ell}$, if $\sigma_\ell$ satisfies (\ref{bound for phi}), then
there holds
\begin{equation}\label{Uniform bound}
  \|\vec{h}_\ell\|_{*,d_\ell}
  \leq
  \left(c(1+2^{d+1})\mathcal R\right)^\ell d_{\ell-1}\cdots
      d_0.
\end{equation}
\end{lemma}

\begin{IEEEproof}
For arbitrary $\ell=1,\dots,L$, it follows  from (\ref{bound for
phi})   that
\begin{eqnarray*}
      &&
      \|\vec{h}_\ell\|_{*,d_\ell}=
       \left\|\vec{\sigma_\ell}(W_\ell \vec{h}_{\ell-1}(x)+\vec{b}_\ell)\right\|_{*,d_\ell}\\
      &=&
      \max_{1\leq i\leq d_\ell}\int_{\mathbb
      I^d}\left|\sigma_\ell(W^{i}_{\ell}\cdot
      \vec{h}_{\ell-1}(x)+b_{\ell}^i)\right|dx\\
      &\leq&
      c \max_{1\leq i\leq d_\ell}\int_{\mathbb
      I^d}\left(\left|\sum_{j=1}^{d_{\ell-1}} W^{i,j}_{\ell}
       {h}^j_{\ell-1}(x)+b^i_{\ell}\right|+1\right)dx\\
      &\leq&
      c\max_{1\leq i\leq d_\ell}\left\{
      \sum_{j=1}^{d_{\ell-1}} |W^{i,j}_{\ell}|
      \int_{\mathbb
      I^d}| {h}^j_{\ell-1}(x)|dx+|b^i_{\ell}|2^d \right\}+c2^d\\
      &\leq&
      (c d_{\ell-1}\|\vec{h}_{\ell-1}\|_{*,d_{\ell-1}}+c2^d) \mathcal
      R+c2^d,
\end{eqnarray*}
where $W^i_\ell$ denotes the $i$-row of the matrix $W_\ell$, $W^{i,j}_\ell$
denotes the $(i,j)$-element of $W_\ell$,
$\vec{b}_\ell=(b_{\ell}^1,\dots,b^{d_\ell}_\ell)$ and $\vec{h}_{\ell-1}\in
\vec{\mathcal H}_{\ell-1}$. Noting $
      \|\vec{h}_0\|_{*,d_0}=\max_{1 \leq i \leq d_0}\int_{\mathbb I} |x^{i}|dx^{i}=1,
$ we then have
$$
  \|\vec{h}_\ell\|_{*,d_\ell}
  \leq
  \left(c(1+2^{d+1})\mathcal R\right)^\ell d_{\ell-1}\cdots
      d_0.
$$
This finishes the proof of Lemma \ref{Lemma:bounds}.
\end{IEEEproof}

Our second lemma aims at deriving covering number of some matrix and
vector with fixed free parameters.

\begin{lemma}\label{Lemma:covering for matrix}
For arbitrary $\varepsilon>0$ and $1\leq \ell\leq L$, we have
$$
    \mathcal N(\varepsilon,\mathcal W_{\mathcal F_{\ell,w}},\|\cdot\|_1)
    \leq \left(\frac{2d_\ell d_{\ell-1}\mathcal R}{\varepsilon}\right)^{\mathcal F_{\ell,w}},
$$
and
$$
       \mathcal N(\varepsilon, \vec{\mathcal
       B}_{\mathcal F_{\ell,b}},\ell_\infty^m)\leq\left(\frac{2\mathcal R }{\varepsilon}\right)^{\mathcal F_{\ell,b}},
$$
where  $
      \|W_{\ell}\|_1:=\sum_{i=1}^{d_{\ell}}\sum_{j=1}^{d_{\ell-1}}|W_{\ell}^{i,j}|
$ denotes the 1-norm of the matrix $W_{\ell}$.
\end{lemma}

\begin{IEEEproof}
For arbitrary $d_{\ell}\times d_{\ell-1}$ matrix, we can rewrite it
as a $d_{\ell}\times d_{\ell-1}$-dimensional vector  as
$\{w_1,\dots,w_{d_{\ell}\times d_{\ell-1}}\}$.
 Without loss of generality, we assume that the first $\mathcal F_{\ell,w}$
 elements of the $d_{\ell}\times d_{\ell-1}$-dimensional  vector are free parameters.
 Let $\mathcal E_{F_i}$ be the $\varepsilon$-cover nets of
 $\{w_i:|w_i|\leq \mathcal R\}$, that is, for each $|w_i|\leq
 \mathcal R$, there is a $w_i'\in \mathcal E_{F_i}$ such that
$$
     |w_i-w_i'|\leq \varepsilon,\qquad\forall i=1,\dots,\mathcal
     F_{\ell,w}.
$$
 Then, for arbitrary $W,W'\in \mathcal W_{\mathcal
 F_{\ell,w}}$ with $W,W'$ the matrices corresponding to the vector
$(w_1,w_2,\dots,w_{\mathcal F_{\ell,w}},\dots)$ and
$(w_1',w_2',\dots,w_{\mathcal F_{\ell,w}}',\dots)$ respectively,
there holds
\begin{eqnarray*}
     &&\|W-W'\|_1
     =
     \sum_{i=1}^{d_{\ell}}\sum_{j=1}^{d_{\ell-1}}|W^{i,j}-W^{'{i,j}}|\\
     &=&
     \sum_{i=1}^{\mathcal F_{\ell,w}} |w_i-w_i'|+ \sum_{i=\mathcal F_{\ell,w}+1}^{d_\ell d_{\ell-1}}|w_i-w_i'|.
\end{eqnarray*}
If the reminder $d_\ell d_{\ell-1}-\mathcal F_{\ell,w}$ are fixed constants, we have
$\sum_{i=\mathcal F_{\ell,w}+1}^{d_\ell d_{\ell-1}}|w_i-w_i'|=0$. If the weight matrix is generated by the other two ways,
which implies some elements in the remainder $d_\ell-\mathcal
F_{\ell,w}$ terms sharing the same values as some elements in the
first $\mathcal F_{\ell,w}$ terms, then we have
$$
    \sum_{i=\mathcal F_{\ell,w}+1}^{d_\ell d_{\ell-1}}|w_i-w_i'| \leq (d_\ell d_{\ell-1}-\mathcal
      F_{\ell,w})\max_{1\leq i\leq \mathcal F_{\ell,w}}|w_i-w_i'|.
$$
Both cases yield
\begin{eqnarray*}
      \|W-W'\|_1
     \leq
      d_\ell d_{\ell-1}\varepsilon.
\end{eqnarray*}
Hence $\mathcal F_{\ell,w}$ $\varepsilon$-covers for  sets
$\{w_i:|w_i|\leq \mathcal R\}$ with $i=1,\dots,\mathcal F_{\ell,w}$
constitute a $d_\ell d_{\ell-1}\varepsilon$-cover for $\mathcal
W_{\mathcal F_{\ell,w}}$, which together with $|\mathcal E_{F_i}|
\leq \frac{2\mathcal R}{\varepsilon}$, $i=1,2,\dots,\mathcal
F_{\ell,w}$ implies
$$
    \mathcal N(\varepsilon,\mathcal W_{\mathcal F_{\ell,w}},\|\cdot\|_1)
    \leq \left(\frac{2d_\ell d_{\ell-1}\mathcal R}{\varepsilon}\right)^{\mathcal F_{\ell,w}},
$$
where $|\mathcal E|$ denotes the cardinality of the set $\mathcal
E$. This completes the first estimate. The second estimates can be
derived by using the similar approach. With these, we completes the
proof of
  Lemma \ref{Lemma:covering for
matrix}.
\end{IEEEproof}

 Based on the previous lemmas, we can derive the
following iterative estimates for the covering number associated
with the affine mapping $\vec{\sigma}(W\vec{h}+\vec{b})$.

\begin{lemma}\label{Lemma:Iterated covering number}
 If $\sigma_\ell$ satisfies Assumption
\ref{Assumption:activation} for each $\ell=1,2,\dots,L$, then
\begin{eqnarray*}
   &&\mathcal N(\varepsilon,\vec{\mathcal
   H}_\ell,\|\cdot\|_{*,d_\ell})
   \leq
     \left(c_1'\mathcal R\right)^{\ell\mathcal F_\ell}D^{2\mathcal F_\ell}_\ell
       \varepsilon^{-\mathcal F_\ell}\\
       &\times&\mathcal N\left(\frac{\varepsilon}{ \left(c_1'\mathcal
R \right)^{\ell-1}D_\ell},\vec{\mathcal
      H}_{\ell-1},\|\cdot\|_{*,d_{\ell-1}}\right),
\end{eqnarray*}
holds for $\ell=2,\dots,L$ and
\begin{eqnarray*}
   \mathcal N( \varepsilon,\vec{\mathcal
       H}_1,\|\cdot\|_{*,d_1})
      \leq
    \left(\frac{  c_1'\mathcal R D_1}{\varepsilon}\right)^{\mathcal F_1}.
\end{eqnarray*}
where $D_\ell=d_\ell\cdots d_0$, $\mathcal F_\ell=\mathcal
F_{\ell,w}+\mathcal F_{\ell,b}$ and $c_1'=6c_1c(1+2^{d+1})$.
\end{lemma}

\begin{proof} For each $\ell=1, 2,\dots,L$, let $\mathcal E_{\ell,w}$ and $\mathcal E_{\ell,b}$
be   $\varepsilon$-cover nets of $\mathcal W_{\mathcal F_{\ell,w}}$
and $\vec{\mathcal B}_{\mathcal F_{\ell,b}}$ respectively. For
$\ell=2,\dots,L$, let $ \mathcal E_{\ell,h}$ be the
$\varepsilon$-cover nets for  $\vec{\mathcal H}_{\ell-1}$.
Therefore, for each $\vec{h}_{\ell-1}\in\vec{\mathcal H}_{\ell-1},
W_\ell\in\mathcal W_{\mathcal F_{\ell,w}}$ and $\vec{b}_\ell
\in\vec{\mathcal B}_{\mathcal F_{\ell,b}}$, there exist
$\vec{h}'_{\ell-1}\in \mathcal E_{\ell,h},W'_\ell \in \mathcal
E_{\ell,w}, \vec{b}'_\ell\in\mathcal E_{\ell,b}$ such that
\begin{equation}\label{varepsilon net}
     \|\vec{h}_{\ell-1}-\vec{h}'_{\ell-1}\|_{*,d_{\ell-1}}\leq\varepsilon,\qquad\ell=2,\dots,L
\end{equation}
and
\begin{equation}\label{varepsilon net1}
    \|W_\ell-W'_\ell\|_1\leq\varepsilon,\quad
    \|\vec{b}_\ell-\vec{b}'_\ell\|_{\ell_\infty^{d_\ell}}\leq\varepsilon,\qquad
    \ell=1,\dots,L
\end{equation}
  Then, for arbitrary $\vec{h}_\ell\in\vec{\mathcal H}_\ell$ and
  $\ell=2,3,\dots,L$, there holds
\begin{eqnarray}\label{comput cov}
   &&
    \|\vec{h}_\ell-\vec{\sigma_\ell}(W'_\ell \vec{h}'_{\ell-1}+\vec{b}'_\ell)\|_{*,d_\ell} \\
    &\leq&
    \|\vec{\sigma_\ell}(W_\ell \vec{h_{\ell-1}}+\vec{b}_\ell)-\vec{\sigma_\ell}(W'_\ell \vec{h}_{\ell-1}+\vec{b}_\ell)\|_{*,d_\ell} \nonumber\\
    &+&
    \|\vec{\sigma_\ell}(W'_\ell \vec{h}_{\ell-1}+\vec{b}_\ell)-\vec{\sigma_\ell}(W'_\ell \vec{h}'_{\ell-1}+\vec{b}_\ell)\|_{*,d_\ell}\nonumber\\
    &+&
    \|\vec{\sigma_\ell}(W'_\ell \vec{h}'_{\ell-1}+\vec{b}_\ell)-\vec{\sigma_\ell}(W'_\ell \vec{h}'_{\ell-1}+\vec{b}'_\ell)\|_{*,d_\ell}.\nonumber
\end{eqnarray}
Due to (\ref{Lip for phi}), we get from Lemma \ref{Lemma:bounds}
that
\begin{eqnarray*}
         &&\|\vec{\sigma_\ell}(W_\ell \vec{h}_{\ell-1}+\vec{b}_\ell)-\vec{\sigma_\ell}(W'_\ell \vec{h}_{\ell-1}+\vec{b}_\ell)\|_{*,d_\ell}\\
         &\leq&
          \max_{1\leq i\leq d_\ell}\int_{\mathbb I^d} |\sigma_\ell (W_\ell^i
          \cdot\vec{h}_{\ell-1}(x)+{b}_{\ell}^i)\\
          &-&\sigma_\ell
         (W^{'i}_\ell\cdot\vec{h}_{\ell-1}(x)+{b}_\ell^i)|dx \\
         &\leq&
         c_1\sum_{i=1}^{d_\ell} \int_{\mathbb I^d}
         |(W_\ell^i-W_\ell^{'i})\cdot\vec{h}_{\ell-1}(x)|dx\\
         &\leq&
         c_1 \|\vec{h}_{\ell-1}\|_{*,d_\ell-1}\sum_{i=1}^{d_\ell}
          \sum_{j=1}^{d_\ell-1}|W_\ell^{ij}-W_\ell^{'ij}|
          \\
         &\leq&
         c_1 \left(c(1+2^{d+1})\mathcal R\right)^{\ell-1} d_{\ell-2}\cdots
      d_0  \|W_{\ell}-W_{\ell}'\|_1.
\end{eqnarray*}
For $\ell=1,2,\dots,L$, we have from (\ref{Lip for phi}) that
\begin{eqnarray*}
         &&\|\vec{\sigma_\ell}(W'_\ell\vec{h}_{\ell-1}+\vec{b}_\ell)-\vec{\sigma_\ell}(W'_\ell\vec{h}'_{\ell-1}+\vec{b}_\ell)\|_{*,d_\ell}\\
         &\leq&
          \max_{1\leq i\leq d_\ell}\int_{\mathbb I^d}
          |\sigma_\ell(W_\ell^{'i}\cdot\vec{h}_{\ell-1}(x)+{b}_\ell^i)\\
          &-&
          \sigma_\ell
         (W_\ell^{'i}\cdot\vec{h}'_{\ell-1}(x)+{b}_\ell^i)|dx \\
         &\leq&
         c_1\max_{1\leq i\leq d_\ell} \int_{\mathbb I^d}
         |W_\ell^{'i}\cdot(\vec{h}_{\ell-1}(x)-\vec{h}'_{\ell-1}(x))|dx\\
         &\leq&
         c_1 \sum_{i=1}^{d_\ell}
          \sum_{j=1}^{d_{\ell-1}}|W_\ell^{'ij}|\|\vec{h}_{\ell-1}-\vec{h}'_{\ell-1}\|_{*,d_{\ell-1}}\\
         &\leq&
         c_1 d_\ell d_{\ell-1} \mathcal R \|\vec{h}_{\ell-1}-\vec{h}'_{\ell-1}\|_{*,d_{\ell-1}}
\end{eqnarray*}
and
\begin{eqnarray*}
         &&\|\vec{\sigma_\ell}(W'_\ell \vec{h}'_{\ell-1}+\vec{b}_\ell)
         -\vec{\sigma_\ell}(W'_\ell \vec{h}'_{\ell-1}+\vec{b}'_\ell)\|_{*,d_\ell}\\
         &=&
          \max_{1\leq i\leq d_\ell} \int_{\mathbb I^d} |\sigma_\ell
          (W_\ell^{'i}\cdot\vec{h}'_{\ell-1}(x)+{b}_\ell^i)\\
          &-&
          \sigma_\ell
         (W_\ell^{'i}\cdot\vec{h}'_{\ell-1}(x)+{b}_\ell^{'i})|dx \\
         &\leq&
         c_1\max_{1\leq i\leq d_\ell}\int_{\mathbb I^d}
         |b_\ell^i-b_\ell^{'i}|dx
         \leq
          2^dc_1
          \|\vec{b}_\ell-\vec{b}'_\ell\|_{\ell_\infty^{d_\ell}}.
\end{eqnarray*}
For $\ell=2,\dots,L$, plugging the above three estimates into
(\ref{comput cov}), we then get from (\ref{varepsilon net}) and
(\ref{varepsilon net1})  that
\begin{eqnarray*}
    &&\|\vec{h}_\ell-\vec{\sigma_\ell}(W'_\ell \vec{h}'_{\ell-1}+\vec{b}'_\ell)\|_{*,d_\ell}\\
   &\leq&
   3c_1(c(1+2^{d+1})\mathcal R)^{\ell-1}d_\ell\dots d_0\varepsilon.
\end{eqnarray*}
This implies that
$$
   \{\vec{\sigma_\ell}(W'_\ell\vec{h}'_\ell+\vec{b}'_\ell):W'_\ell
   \in\mathcal E_{\ell,w},
   \vec{h}'_{\ell-1}\in\mathcal E_{\ell,h},\vec{b}'_\ell\in\mathcal E_{\ell,b}\}
$$
is a  $ 3c_1(2c(1+2^{d+1})\mathcal R)^{\ell-1}d_\ell\dots
d_0\varepsilon$-net of $\vec{\mathcal H}_{\ell}$. This together with
Lemma \ref{Lemma:covering for matrix} implies
\begin{eqnarray*}
      &&\mathcal N\left( 3c_1(c(1+2^{d+1})\mathcal R)^{\ell-1}d_\ell\dots d_0\varepsilon,\vec{\mathcal H}_\ell,\|\cdot\|_{*,d_\ell}\right)\\
      &\leq&
    \left(\frac{  2d_\ell d_{\ell-1}\mathcal R}{\varepsilon}\right)^{\mathcal F_{\ell,w}+\mathcal F_{\ell,b}}
      \mathcal N\left(\varepsilon,\vec{\mathcal
      H}_{\ell-1},\|\cdot\|_{*,d_{\ell-1}}\right).
\end{eqnarray*}
Scaling  $\varepsilon$ to $ \frac{\varepsilon}{3c_1(c(1+2^{d+1})\mathcal
R)^{\ell-1}d_\ell\dots d_0} $, we then have
\begin{eqnarray*}
   &&\mathcal N(\varepsilon,\vec{\mathcal
   H}_\ell,\|\cdot\|_{*,d_\ell})
   \leq
     \left((c_1'\mathcal R)\right)^{\ell\mathcal F_\ell}D^{2\mathcal F_\ell}_\ell
       \varepsilon^{-\mathcal F_\ell}\\
       &\times&\mathcal N\left(\frac{\varepsilon}{ \left(c_1'\mathcal
R \right)^{\ell-1}D_\ell},\vec{\mathcal
      H}_{\ell-1},\|\cdot\|_{*,d_{\ell-1}}\right),
\end{eqnarray*}
where $c_1'=6c_1c(1+2^{d+1})$ and $D_\ell:=d_\ell\dots d_0$. This proves
  Lemma \ref{Lemma:Iterated covering
number} for $\ell=2,\dots,L$. If $\ell=1$, then for arbitrary
$\vec{h}_1\in\vec{\mathcal H}_1$, we have
\begin{eqnarray}\label{comput cov aa}
   &&
    \|\vec{h}_1-\vec{\sigma_1}(W'_1 x+\vec{b}'_1)\|_{*,d_1} \nonumber\\
    &\leq&
    \|\vec{\sigma_1}(W_1x+\vec{b}_1)-\vec{\sigma_1}(W'_1x+\vec{b}_1)\|_{*,d_1} \nonumber\\
    &+&
    \|\vec{\sigma_1}(W'_1x+\vec{b}_1)-\vec{\sigma_1}(W'_1x+\vec{b}'_1)\|_{*,d_1}.
\end{eqnarray}
The same approach as above yields  that
$$
   \{\vec{\sigma_1}(W'_1x+\vec{b}'_1):W'_1\in\mathcal E_{1,w},
    \vec{b}'_1\in\mathcal E_{1,b}\}
$$
is a  $c_12^{d+1}\varepsilon$-net of $\vec{\mathcal H}_1$. Using
Lemma \ref{Lemma:covering for matrix} again, we
obtain
\begin{eqnarray*}
       \mathcal N( c_12^{d+1}\varepsilon,\vec{\mathcal
       H}_1,\|\cdot\|_{*,d_1})
      \leq
    \left(\frac{  2D_1\mathcal R}{\varepsilon}\right)^{\mathcal F_{1,w}+\mathcal F_{1,b}}.
\end{eqnarray*}
Scaling  $\varepsilon$ to $\varepsilon/ c_12^{d+1}  $,    we get
\begin{eqnarray*}
   \mathcal N( \varepsilon,\vec{\mathcal
       H}_1,\|\cdot\|_{*,d_1})
      \leq
    \left(\frac{  c_1'\mathcal R D_1}{\varepsilon}\right)^{\mathcal F_1}.
\end{eqnarray*}
This completes the proof of Lemma \ref{Lemma:Iterated covering
number}.
\end{proof}

With the help of the above two lemmas, we are in a position to prove
Theorem \ref{Theorem:covering number}.

\begin{proof}[Proof of Theorem \ref{Theorem:covering number}]
Let $\vec{A}^*_{\mathcal F_{L,a}}$ be the set of $  d_L$-dimensional
vectors with fixed structures and totally $\mathcal F_{L,a}$ free
parameters. Denote $\vec{\mathcal A}_{\mathcal
F_{L,a}}:=\{\vec{a}\in \vec{A}^*_{\mathcal F_{L,a}}: |a_{i}|\leq
\mathcal R\}$.
 Assume
 that $\mathcal E_{L,a}$ is an
$\varepsilon$-cover of the set $\vec{\mathcal A}_{\mathcal F_{L,a}}$
under the metric of $\ell_1^{d_L}$.  Then,  for arbitrary
$\vec{a}\in \mathbb R^{d_L}$ and $\vec{h}_{L}\in \vec{\mathcal
H}_{L}$  there is a $\vec{a}^*\in \mathcal E_{L,a}$ and
$\vec{h}_{L}^*\in \mathcal E_{L,h}$ such that
$$
        \|\vec{a}-\vec{a}^*\|_{\ell_1^{d_L}}\leq\varepsilon,\qquad\mbox{and}\quad
        \|\vec{h}_{L}-\vec{h}_{L}^*\|_{*,d_L}\leq\varepsilon.
$$
Note that
\begin{eqnarray}\label{equ: decomposition}
     &&\|\vec{a}\cdot\vec{h}_L -\vec{a}^*\cdot
     \vec{h}_L^*\|_{L^1(\mathbb I^d)}\leq
     \|\vec{a}\cdot\vec{h}_L -\vec{a}^*\cdot
     \vec{h}_L\|_{L^1(\mathbb I^d)} \nonumber\\
     &+&\|\vec{a}^*\cdot\vec{h}_L -\vec{a}^*\cdot
     \vec{h}_L^*\|_{L^1(\mathbb I^d)}.
\end{eqnarray}
Moreover, Lemma \ref{Lemma:bounds} shows
\begin{eqnarray*}
    &&   \|\vec{a}\cdot\vec{h}_L -\vec{a}^*\cdot
     \vec{h}_L\|_{L^1(\mathbb I^d)}\le
     \|\vec{a}-\vec{a}^*\|_{\ell_1^{d_L}}\|\vec{h}_{L}\|_{*,d_L}\\
     &\leq&
     \left(c(1+2^{d+1})\mathcal R\right)^L d_{L-1}\cdots
      d_0 \|\vec{a}-\vec{a}^*\|_{\ell_1^{d_L}}\\
      &\leq&
       \left(c(1+2^{d+1})\mathcal R\right)^L d_{L-1}\cdots
      d_0\varepsilon
\end{eqnarray*}
and
\begin{eqnarray*}
     &&\|\vec{a}^*\cdot\vec{h}_L -\vec{a}^*\cdot
     \vec{h}_L^*\|_{L^1(\mathbb I^d)}\le
     d_L\mathcal{R} \|\vec{h}_L-\vec{h}_L^*\|_{*,d_L}\\
     &\leq&
     d_L\mathcal{R}\varepsilon.
\end{eqnarray*}
Plugging the above estimates  into (\ref{equ: decomposition}), we
have
\begin{eqnarray*}
     \|\vec{a}\cdot\vec{h}_L -\vec{a}^*\cdot
     \vec{h}_L^*\|_{L^1(\mathbb I^d)}
      \leq
       \left(c_2'\mathcal R \right)^LD_L\varepsilon,
\end{eqnarray*}
where $c_2'=2c(1+2^{d+1})$.  Since Lemma \ref{Lemma:covering for matrix}
implies
\begin{eqnarray*}
     &&\mathcal N\left(\frac{\varepsilon}{\left(c_2'\mathcal R \right)^{L}D_L},\vec{\mathcal A}_{\mathcal
     F_{L,a}},\ell_1^{d_L}\right)\\
      &\leq& \left(c_2'\mathcal R\right)^{(L+1)\mathcal
      F_{L,a}}D_L^{2\mathcal F_{L,a}}
      \varepsilon^{-\mathcal F_{L,a}},
\end{eqnarray*}
there holds
\begin{eqnarray}\label{First layer}
     &&\mathcal N\left( \varepsilon,\mathcal H_{n,L,\sigma,\mathcal
       R},L_1(\mathbb I^d)\right)\nonumber\\
     &\leq& \left(c_2'\mathcal R\right)^{(L+1)\mathcal
      F_{L,a}}D_L^{2\mathcal F_{L,a}}
      \varepsilon^{-\mathcal F_{L,a}}\nonumber\\
     &\times& \mathcal N\left(\frac{\varepsilon}{\left(c_2'\mathcal R \right)^LD_L},\vec{\mathcal H}_{L},\|\cdot\|_{*,d_L}\right).
\end{eqnarray}
We then use Lemma \ref{Lemma:Iterated covering number} to estimate
the second part of the above term. Let
\begin{equation}\label{def.b}
\begin{split}
        B_\ell&:=2(\max\{c_1',c_2'\}\mathcal R)^\ell D_\ell^2D_{\ell+1},
        \ell=1,2,\dots,L-1,\\
        B_L&:=2(\max\{c_1',c_2'\}\mathcal R)^L D_L^2,\\
        B_{L+1}&:=2(\max\{c_1',c_2'\}\mathcal R)^{L+1} D_L^2.
\end{split}
\end{equation}
Then, the first estimate of Lemma \ref{Lemma:Iterated covering
number} shows
\begin{eqnarray*}
   &&\mathcal N(\varepsilon,\vec{\mathcal
   H}_\ell,\|\cdot\|_{*,d_\ell})
   \leq
       B_\ell^{\mathcal F_\ell} \varepsilon^{-\mathcal F_\ell} \\
       &\times&\mathcal N\left(  \frac{\varepsilon}{B_{\ell-1}},\vec{\mathcal
      H}_{\ell-1},\|\cdot\|_{*,d_{\ell-1}}\right),\qquad
      \ell=L,\dots,2.
\end{eqnarray*}
Using the above inequality iteratively with $\ell=L,L-1,\dots,2$, we
obtain
\begin{eqnarray*}
   &&\mathcal N\left(   {\frac{\varepsilon}{\left(c_2'\mathcal R \right)^LD_L}},\vec{\mathcal
      H}_{L},\|\cdot\|_{*,d_L}\right)\\
      &\leq&
      \left(\prod_{\ell=2}^{L-1}B_\ell^{\mathcal F_\ell}\right)
      \left(\varepsilon^{-\sum_{\ell=2}^L\mathcal F_\ell}\right)\mathcal N\left(   \frac{\varepsilon}{\prod_{\ell=1}^LB_\ell},\vec{\mathcal
      H}_{1},\|\cdot\|_{*,d_{1}}\right)\\
      &\times&
      B_L^{\mathcal F_L}(B_LB_{L-1})^{\mathcal F_{L-1}}\cdots
      (B_L\cdots B_2)^{\mathcal F_2}\\
     &=&
      \left(\prod_{\ell=2}^{L-1}B_\ell^{\mathcal F_\ell}\right)\left(\prod_{\ell=2}^L
      B_\ell^{\sum_{j=2}^\ell\mathcal F_j}\right)\varepsilon^{-\sum_{\ell=2}^L\mathcal
      F_\ell}\\
      &\times&
      \mathcal N\left(   \frac{\varepsilon}{\prod_{\ell=1}^LB_\ell},\vec{\mathcal
      H}_{1},\|\cdot\|_{*,d_{1}}\right)
\end{eqnarray*}
But the second estimate in Lemma \ref{Lemma:Iterated covering
number} and the definition of $B_\ell$ yield
\begin{eqnarray*}
    \mathcal N\left(
     \frac{\varepsilon}{\prod_{\ell=1}^LB_\ell},\vec{\mathcal
      H}_{1},\|\cdot\|_{*,d_{1}}\right)
       \leq
     \left(  B_1\prod_{\ell=1}^LB_\ell\right)^{\mathcal
     F_1}\varepsilon^{-\mathcal F_1}.
\end{eqnarray*}
Then,
\begin{eqnarray}\label{final e}
   &&\mathcal N\left(   {\frac{\varepsilon}{\left(c_2'\mathcal R \right)^LD_L}},\vec{\mathcal
      H}_{L},\|\cdot\|_{*,d_L}\right) \nonumber \\
      &\leq&
       \left(\prod_{\ell=1}^{L-1}B_\ell^{\mathcal F_\ell}\right) \left(\prod_{\ell=1}^L
      B_\ell^{\sum_{j=1}^\ell\mathcal F_j}\right) \varepsilon^{-\sum_{\ell=1}^L\mathcal
      F_\ell}
\end{eqnarray}
Inserting the above estimate into (\ref{First layer}), we have
\begin{eqnarray*}
 &&\mathcal N\left( \varepsilon,\mathcal H_{n,L,\sigma,\mathcal
       R\}},L_1(\mathbb I^d)\right)\\
     &\leq& B_{L+1}^{\mathcal F_{L,a}}
         \prod_{\ell=1}^L
      B_\ell^{\mathcal F_\ell+\sum_{j=1}^\ell\mathcal F_j}\varepsilon^{-\sum_{\ell=1}^L\mathcal
      F_\ell-\mathcal F_{L,a}}.
\end{eqnarray*}
It follows from (\ref{def.b}) that
$$
        \max_{1\leq \ell\leq L+1}
        B_\ell\leq B_{L+1}D_L.
$$
Then,
\begin{eqnarray*}
 &&\mathcal N\left( \varepsilon,\mathcal H_{n,L,\sigma,\mathcal
       R\}},L_1(\mathbb I^d)\right)\\
     &\leq&
    (B_{L+1}D_L)^{\mathcal F_{L,a}+(L+1)\sum_{\ell=1}^L\mathcal
    F_{\ell}}\varepsilon^{-n}\\
    &\leq&
    (B_{L+1}D_L)^{(L+1)n}\varepsilon^{-n}.
\end{eqnarray*}
This together with (\ref{def.b}) yields
\begin{eqnarray*}
  \mathcal N\left( \varepsilon,\mathcal H_{n,L,\sigma,\mathcal
       R\}},L_1(\mathbb I^d)\right)
     \leq
   \left((c_3\mathcal R)^{L+1}D^3_L\right)^{(L+1)n}\varepsilon^{-
   n},
\end{eqnarray*}
where $c_3=2\max\{c_1',c_2'\}$
 This completes the proof
of Theorem \ref{Theorem:covering number} by noting $D_L\leq
D_{\max}^{L+1}$.
\end{proof}

\section*{Appendix B: Covering Numbers and Approximation}

The main tool in our analysis is
  a relation between   covering numbers
 and lower bounds of approximation, which is presented in the following theorem.

\begin{theorem}\label{theorem:Relation c and l}
Let $n\in\mathbb N$ and $V\subseteq L_1(\mathbb I^d)$. For arbitrary
$\varepsilon>0$, if
\begin{equation}\label{covering condition}
  \mathcal N(\varepsilon,V)\leq
  \tilde{C}_1\left(\frac{\tilde{C_2}n^\beta}{\varepsilon}\right)^n
\end{equation}
with $\beta,\tilde{C}_1,\tilde{C}_2>0$, then
\begin{equation}\label{lower bound deep}
   \mbox{dist}(Lip^{(r,c_0)},V,L_1(\mathbb I^d))\geq C'
       (n\log_2(n+1))^{-r/d},
\end{equation}
where
\begin{eqnarray*}
     C'&:=&\frac14 d^{-d/2}\left[32(1+\beta+3r/d)\left(\log_2(2\tilde{C}_1\right.\right.\\
     &+&
     \left.\left.8d^{d/2}(1+\beta+3r/d+\tilde{C}_2))+1\right)\right]^{-\frac{r}{d}}.
\end{eqnarray*}
\end{theorem}

We postpone the proof of Theorem \ref{theorem:Relation c and l} to
the end of this section. Theorem \ref{theorem:Relation c and l}
shows that to approximate functions in $Lip^{(r,c_0)}$, the capacity
of the approximations, measured by the covering number, plays a
crucial role.  To present the limitations of deep nets,  Theorem
\ref{theorem:Relation c and l} implies that  we only need to
estimate their covering numbers. We highlight that Theorem
\ref{theorem:Relation c and l} is motivated by \cite{Maiorov1999c},
in which a relation between the so-called pseudo-dimensions and
lower bounds of approximation is established. However, estimating
  pseudo-dimensions of classes of functions is not so easy, even
for shallow nets \cite{Maiorov2006}.

 To prove Theorem
\ref{theorem:Relation c and l}, we need the following four technical
lemmas. At first, we introduce  the definition of the
$\varepsilon$-packing number (see \cite{Zhou2002,Zhou2003}) by
\begin{eqnarray*}
   &&\mathcal M(\varepsilon,V,B)\\
   &&=\max\{m:\exists f_1,\dots, f_m\in B,
   \|f_i-f_j\|_B\geq\varepsilon,\forall i\neq j\}.
\end{eqnarray*}
We also denote $\mathcal M(\varepsilon,V):=\mathcal
M(\varepsilon,V,L_1(\mathbb I^d))$. The following lemma which was
proved in \cite[Lemma 9.2]{Gyorfi2002} establishes a relation
between $\mathcal N(\varepsilon,V)$ and $\mathcal M(\varepsilon,V)$.

\begin{lemma}\label{Lemma:covering and packing}
      For arbitrary $\varepsilon>0$ and  $V\subseteq L_1(\mathbb I^d)$, there holds
$$
     \mathcal M(2\varepsilon,V)\leq \mathcal N(\varepsilon,V)\leq \mathcal
     M(\varepsilon,V).
$$
\end{lemma}

For arbitrary $N^*\in \mathbb N$, denote
$E^{(N^*)^d}:=\{\epsilon=(\epsilon_1,\dots,\epsilon_{(N^*)^d}):
\epsilon_i\in\{-1,1\},1\leq i\leq (N^*)^d\}$. The following lemma
can be found in \cite[P.489]{Lorentz1996} (see also\cite[Claim
1]{Maiorov1999c}).

\begin{lemma}\label{Lemma: Tools for vc}
For arbitrary $N^*\in \mathbb N$, there exists a set
$G^{(N^*)^d}\subset E^{(N^*)^d}$ with  $|G^{(N^*)^d}|\geq
2^{(N^*)^d/16}$ such that for any $v,v'\in G^{(N^*)^d}$ with $v\neq
v'$, there holds $\|v-v'\|_{\ell_1}\geq (N^*)^d/2$, where
$\|v\|_{\ell_1}=\sum_{i=1}^{(N^*)^d}|v_i|$ for
$v=(v_1,\dots,v_{(N^*)^d})$ and $|G^{(N^*)^d}|$ denotes the
cardinality of $G^{(N^*)^d}$.
\end{lemma}

Define $g:\mathbb R\rightarrow\mathbb R$ such that
$supp(g)\subseteq[-1/\sqrt{d},1/\sqrt{d}]^d$, $g(x)=1$ for
$x\in[-1/(2\sqrt{d}),1/(2\sqrt{d})]^d$ and $g\in
Lip^{(r,c_02^{v-1})}$, where $supp(g)$ denotes the support of $g$.
Partition $\mathbb I^d$ by ${(N^*)^d}$ sub-cubes
$\{A_k\}_{k=1}^{(N^*)^d}$ of side length $ 1/N^*$ and with centers
$\{\xi_k\}_{k=1}^{(N^*)^d}$. For arbitrary $x\in\mathbb I^d$, define
\begin{equation}\label{Def.gk}
          g_k(x):=({N^*})^{-r}g({N^*}(x-\xi_k))
\end{equation}
and
\begin{equation}\label{subset}
     \mathcal F_{G^{(N^*)^d}}:=\left\{
     \sum_{k=1}^{(N^*)^d}\epsilon_kg_k(x):\epsilon=(\epsilon_1,\dots,\epsilon_{(N^*)^d})\in
      G^{(N^*)^d}\right\}.
\end{equation}
The following lemma shows that $\mathcal F_{G^{(N^*)^d}}\subset
Lip^{(r,c_0)}$

\begin{lemma}\label{Lemma:Subset}
For arbitrary $N^*\in\mathbb N$,   we have
$$
     \mathcal F_{ G^{(N^*)^d}}\subset Lip^{(r,c_0)},
$$
where $G^{(N^*)^d}$ is defined in Lemma \ref{Lemma: Tools for vc}.
\end{lemma}

\begin{proof} Let $\vec{\alpha}=(\alpha_1,\cdots,\alpha_d)$. Denote by
$$
         f^{(\vec{\alpha})}(x)=\frac{\partial^sf}{\partial x_1^{\alpha_1}\dots\partial
          x_d^{\alpha_d}}
          (x)
$$
for every $\alpha_j\in \mathbb N_0$, $j=1,\dots,d$ with
$\alpha_1+\dots+\alpha_d=s$. Since
\begin{equation}\label{orthogonal tools}
          \| {N^*} (x-\xi_k)-{N^*} (x-\xi_{k'})\|
          = N^* \|\xi_k-\xi_{k'}\|\geq1,\qquad \forall\ k\neq k',
\end{equation}
 $ {N^*}  (x-\xi_k)$ and $ {N^*} (x-\xi_{k'})$ do not belong to the
 set
$(-1/\sqrt{d},1/\sqrt{d})^d$ simultaneously. Then it follows from
$supp(g)\subseteq [-1/\sqrt{d},1/\sqrt{d}]^d$ that for arbitrary
$x\in\mathbb I^d$, there is at most one $k\in\{1,2,\dots,(N^*)^d\}$
such that $ g_k(x)\neq0, g^{(\vec{\alpha})}_{k}(x)\neq 0$, that is,
\begin{equation}\label{localization}
     g_k(x)=0,\ g^{(\vec{\alpha})}_{k}(x)=0,\qquad \mbox{if}\ x\in A_{k'}\ \mbox{with}\ k'\neq k.
\end{equation}
 If $x,x'\in A_{k_0}$ for some
$k_0\in\{1,\dots,(N^*)^d\}$,
       then
        $g_{k}(x)=0$ for $k\neq k_0$. So, for each  $f\in\mathcal F_{G^{(N^*)^d}}$, we  get from
$|\epsilon_k|=1$, (\ref{Def.gk}) and $g\in Lip^{(r,c_02^{v-1})}$
with $r=s+v$ and $0<v\leq 1$ that
\begin{eqnarray*}
      &&|f^{(\vec{\alpha})}(x)-f^{(\vec{\alpha})}(x')|
      =\left|\sum_{k=1}^{(N^*)^d}\epsilon_{ k}[g^{(\vec{\alpha})}_{k}(x)-g^{(\vec{\alpha})}_{
      k}(x')]\right|\\
      &=&|g^{(\vec{\alpha})}_{k_0}(x)-g^{(\vec{\alpha})}_{
      k_0}(x')|
      \\
       &=&
       (N^*)^{-r+s}\left|
       [g^{(\vec{\alpha})}(N^*(x-\xi_{k_0})-g^{(\vec{\alpha})}(N^*(x'-\xi_{k_0})]\right|\\
       &\leq&
      c_0 2^{v-1}\|x-x'\|^v
       \leq
     c_0\|x-x'\|^v.
\end{eqnarray*}
If $x\in A_{k_1}$ but $x'\in A_{k_2}$ for some
$k_1,k_2\in\{1,\dots,(N^*)^d\}$  with $k_1\neq k_2$, we
      can
choose $z\in \partial A_{k_1}$ and $z'\in \partial A_{k_2}$ such
that $z,z'$ are on the segment between $x$ and ¡¡$x'$, where
$\partial A$ denotes the boundary of the sub-cube $A$. Then
$$
           \|x-z\|+\|x'-z'\|\leq \|x-x'\|.
$$
Due to the fact that $supp(g)\subseteq [-1/\sqrt{d},1/\sqrt{d}]^d$, $g$ is smooth
on $\mathbb R^d$ and (\ref{Def.gk}), we get
\begin{equation}\label{value of boundary}
    g^{(\vec{\alpha})}_{k_1}(z)=g^{(\vec{\alpha})}_{k_2}(z')=0.
\end{equation}
So,  $g\in Lip^{(r,c_02^{v-1})}$ with $0<v\leq 1$  and Jensen's
inequality yield
\begin{eqnarray*}
        &&|f^{(\vec{\alpha})}(x)-f^{(\vec{\alpha})}(x')|
      =\left|\sum_{k=1}^{(N^*)^d}\epsilon_{k}[g^{(\vec{\alpha})}_{ k}(x)-g^{(\vec{\alpha})}_{
      k}(x')]\right|\\
       &\leq&
      \left|g^{(\vec{\alpha})}_{k_1}(x)\right|+\left|g^{(\vec{\alpha})}_{k_2}(x')\right|\\
      &=&
      \left|g^{(\vec{\alpha})}_{k_1}(x)-g^{(\vec{\alpha})}_{k_1}(z)\right|+\left|g^{(\vec{\alpha})}_{k_2}(x')- g^{(\vec{\alpha})}_{k_2}(z')
      \right|\\
       &\leq&
      (N^*)^{d-r}\left[|g^{(\vec{\alpha})}( N^* (x-\xi_{k_1}))-g^{(\vec{\alpha})}( N^*
      (z-\xi_{k_1}))|\right.\\
       &+&
      \left.|g^{(\vec{\alpha})}(N^* (x'-\xi_{k_2}))-g^{(\vec{\alpha})}( N^* (z'-\xi_{k_2}))|\right]\\
      &\leq&
      c_02^v\left[\frac{\|x-z\|^v}2+\frac{\|x'-z'\|^v}2\right]\\
      &\leq&
       c_02^v\left[\frac{\|x-z\|}2+\frac{\|x'-z'\|}2\right]^v\leq
       c_0\|x-x'\|^v.
\end{eqnarray*}
Both assertions yield  $f\in Lip^{(r,c_0)}$ and proves Lemma
\ref{Lemma:Subset}
\end{proof}

The last lemma describes the geometry of  $\mathcal
F_{G^{(N^*)^d}}$.

\begin{lemma}\label{Lemma:distance}
Let $N^*\in\mathbb N$ and $G^{(N^*)^d}$ be defined in Lemma
\ref{Lemma: Tools for vc}. For any $f\neq f_1\in\mathcal
F_{G^{(N^*)^d}}$, there holds
\begin{equation}\label{distance111}
           \|f-f_1\|_{L_1(\mathbb I^d)}\geq
           \frac12d^{-d/2}(N^*)^{-r}.
\end{equation}
\end{lemma}

\begin{proof}
 For arbitrary $f,f_1\in \mathcal F_{ G^{(N^*)^d}}$ with $f\neq
f_1$, it follows from (\ref{subset}) that there exist
$\epsilon,\epsilon'\in G^{(N^*)^d}$ with $\epsilon\neq\epsilon'$
such that
\begin{eqnarray}\label{integral 1}
   \|f-f_1\|_{L_1(\mathbb I^d)}
    =\int_{\mathbb I^d}\left|\sum_{k=1}^{(N^*)^d}(\epsilon_k-\epsilon_k')g_k(x)\right|dx.
\end{eqnarray}
Since $g_k(x)=0$ for $x\in\partial A_{k}$, $k=1,2,\dots,(N^*)^d$, we
get from (\ref{integral 1}), (\ref{Def.gk}), (\ref{localization})
and $g\in Lip^{(r,c_02^{v-1})}$   that
\begin{eqnarray}\label{integral2}
  &&\|f-f_1\|_{L_1(\mathbb I^d)}
   =
  \sum_{k'=1}^{(N^*)^d}
  \int_{A_{k'}}\left|\sum_{k=1}^{(N^*)^d}(\epsilon_k-\epsilon_k')g_k(x)\right|dx \nonumber \\
 & =&
   \sum_{k'=1}^{(N^*)^d}
  \int_{A_{k'}}\left| (\epsilon_{k'}-\epsilon_{k'}')g_{k'}(x)\right|dx
     \nonumber \\
  &=&
  (N^*)^{-r}\sum_{k'=1}^{(N^*)^d}|\epsilon_{k'}-\epsilon_{k'}'|
  \int_{A_{k'}}\left|  g(N^*(x-\xi_{k'})) \right|dx.
\end{eqnarray}
For each $k'=1,2,\dots,(N^*)^d$, when $x$ runs over $A_{k'}$,
$N^*(x-\xi_{k'})$ runs over a cube $S$ centered at $\xi_{k'}$ and
with side-length $1$.  Then, $g(x)=1$ for
$x\in[-1/(2\sqrt{d}),1/(2\sqrt{d})]^d$ yields
\begin{eqnarray}\label{integral3}
    &&\int_{A_{k'}}\left|  g(N^*(x-\xi_{k'})) \right|dx\\
    &=&\int_{A_{k'}}\left|  g(N^*(x-\xi_{k'}))
    \right|d(x-\xi_{k'})\nonumber\\
    &\geq&(N^*)^{-d}\int_{[-1/(2\sqrt{d}),1/(2\sqrt{d})]^d}|g(x)|dx\geq (\sqrt{d}N^*)^{-d}.\nonumber
\end{eqnarray}
But Lemma \ref{Lemma: Tools for vc} with $\epsilon,\epsilon'\in
G^{(N^*)^d}$ shows
\begin{equation}\label{integral4}
     \sum_{k'=1}^{(N^*)^d}|\epsilon_{k'}-\epsilon_{k'}'|\geq
     (N^*)^d/2.
\end{equation}
Hence, for arbitrary $f,f_1\in\mathcal F_{ G^{(N^*)^d}}$, inserting
(\ref{integral3}) and (\ref{integral4}) into (\ref{integral2}), we
obtain
\begin{equation}\label{integral 5}
   \|f-f_1\|_{L_1(\mathbb I^d)}\geq
   (N^*)^{-r}(\sqrt{d}N^*)^{-d}(N^*)^d/2\geq
   \frac12d^{-d/2}(N^*)^{-r}.
\end{equation}
This completes the proof of Lemma \ref{Lemma:distance}.
\end{proof}

By the help of the above four lemmas, we  are in a position to prove
Theorem \ref{theorem:Relation c and l}.

\begin{proof}[Proof of Theorem \ref{theorem:Relation c and l}] For arbitrary $\nu>0$, denote
\begin{equation}\label{def delta}
    \delta=\mbox{dist}(\mathcal
    F_{ G^{(N^*)^d}},V,L_1(\mathbb I^d))+\nu.
\end{equation}
For any $f\in\mathcal F_{ G^{(N^*)^d}},$ define a function $Pf\in V$
  such that
\begin{equation}\label{projection}
        \|f-Pf\|_{L_1(\mathbb
        I^d)}\leq\delta.
\end{equation}
Due to (\ref{def delta}), there are more than one  $Pf$ satisfying
(\ref{projection}).  Define   $\mathcal T_{
G^{(N^*)^d}}:=\{Pf:f\in\mathcal F_{ G^{(N^*)^d}}\}\subseteq V$. For
arbitrary   $f,f_1\in\mathcal F_{ G^{(N^*)^d}}$ with $f\neq f_1$,
write  $f^*=Pf$ and $f^*_1=Pf_1$. Then
\begin{eqnarray*}
     &&\|f^*-f_1^*\|_{L_1(\mathbb I^d)}
      =
     \|Pf-Pf_1\|_{L_1(\mathbb I^d)}\\
     &=&
     \|Pf-f+f-f_1+f_1-Pf_1\|_{L_1(\mathbb I^d)}\\
     &\geq&
     \|f-f_1\|_{L_1(\mathbb I^d)}
     -
     \|Pf-f\|_{L_1(\mathbb I^d)}
     -\|Pf_1-f_1\|_{L_1(\mathbb I^d)},
\end{eqnarray*}
which together with  (\ref{distance111}) and (\ref{def delta}) shows
\begin{equation}\label{1.lower1}
   \|f^*-f_1^*\|_{L_1(\mathbb I^d)}\geq \frac12d^{-d/2}(N^*)^{-r}-2\delta.
\end{equation}
We now claim $\delta> \frac18d^{-d/2}(N^*)^{-r}$ for $N^*$
satisfying
\begin{eqnarray}\label{Def.N*}
    &&(N^*)^d=\left\lceil32(1+\beta+3r/d)n\right. \nonumber\\
    &\times&\left.\log_2(2\tilde{C}_1+8d^{d/2}(1+\beta+3r/d+\tilde{C_2})+n)\right\rceil,
\end{eqnarray}
with $\lceil a\rceil$ denoting the smallest integer not smaller than
 the positive number $a$.
To prove the claim, suppose, to the contrary, that
\begin{equation}\label{condration condition}
     \delta\leq \frac18d^{-d/2}(N^*)^{-r},
\end{equation}
 then (\ref{1.lower1}) implies
$$
        \|f^*-f^*_1\|_{L_1(\mathbb I^d)}\geq \frac14d^{-d/2}(N^*)^{-r}.
$$
This shows that $f\neq f_1$ implies $f^*\neq f^*_1$.  So it follows
from Lemma \ref{Lemma: Tools for vc} that
$$
     |\mathcal T_{ G^{(N^*)^d}}|=|\mathcal F_{ G^{(N^*)^d}}| =|G^{(N^*)^d}| \geq 2^{(N^*)^d/16}.
$$
 Fixing
$\varepsilon_0=\frac14d^{-d/2}(N^*)^{-r}$, we then have
$$
    \mathcal
     M(\varepsilon_0,V)\geq 2^{(N^*)^d/16}.
$$
On the other hand, since $\mathcal T_{ G^{(N^*)^d}}\subseteq V$, it
follows from (\ref{covering condition}) and Lemma
\ref{Lemma:covering and packing} that
\begin{eqnarray*}
      \mathcal
     M(\varepsilon_0,V)&\leq& \mathcal
     N(\varepsilon_0/2,V)\leq
     \tilde{C}_1\left(\frac{2\tilde{C}_2n^\beta}{\varepsilon_0}\right)^n\\
     &=&\tilde{C}_1\left(2\tilde{C}_2n^\beta4d^{d/2}(N^*)^r\right)^n.
\end{eqnarray*}
Combining the above two inequalities, we get
\begin{equation}\label{number packing}
        2^{(N^*)^d/16}\leq
        \tilde{C}_1\left(2\tilde{C}_2n^\beta4d^{d/2}(N^*)^r\right)^n.
\end{equation}
This together with (\ref{Def.N*}) shows
\begin{eqnarray}\label{false inequality}
    &&2(1+\beta+3r/d)n \nonumber\\
    &\times&\log_2(2\tilde{C}_1+8d^{d/2}(1+\beta+3r/d+\tilde{C}_2)+n) \nonumber\\
     &<&
      \log_2(\tilde{C}_1)
      +
     n\log_2(\tilde{C}_24d^{d/2}) \\
     &+&
     \beta n\log n+ \frac{rn}d\log_2(32(\beta+1+3r/d))+ \frac{rn}d\log_2 n\nonumber\\
     &+&
     \frac{rn}{d}\log_2\log_2(2\tilde{C}_1+8d^{d/2}(1+\beta+3r/d+\tilde{C_2})+n).\nonumber
\end{eqnarray}
Since the righthand of the above inequality is smaller than
$$
       (2+\beta+3r/d)n\log_2(2\tilde{C}_1+8d^{d/2}(1+\beta+3r/d+\tilde{C}_2)+n),
$$
it leads to a contradiction.  This proves the claim $\delta>
\frac18d^{-d/2}(N^*)^{-r}$  for $N^*$ satisfying (\ref{Def.N*}).
 Noting for arbitrary $u\geq 2$,
$$
    \log_2(n+u)\leq \log_2u+\log_2(n+1)\leq (\log_2 u+1)\log_2(n+1),
$$
we   have
$$
       \delta>\frac18 d^{-d/2}(N^*)^{-r}\geq
       2C'
       (n\log_2(n+1))^{-r/d}.
$$
But (\ref{def delta}) with
$\nu=\delta/2$ shows
$$
   \mbox{dist}(\mathcal
    F_{ G^{(N^*)^d}},V,L_1(\mathbb I^d))=\frac\delta2>C'
       (n\log_2(n+1))^{-r/d}.
$$
Therefore, it follows from Lemma \ref{Lemma:Subset} that
\begin{eqnarray*}
   &&\mbox{dist}(Lip^{(r,c_0)},V,L_1(\mathbb I^d))\geq\mbox{dist}(\mathcal
    F_{ G^{(N^*)^d}},V,L_1(\mathbb I^d))\\
    &\geq& C'
       (n\log_2(n+1))^{-r/d}.
\end{eqnarray*}
  This
completes the proof of Theorem \ref{theorem:Relation c and l}.
\end{proof}

\section*{Appendix C: Proof of Theorem \ref{Theorem:lower bound for deep
nets} }

Combining Theorem \ref{Theorem:covering number} and Theorem
\ref{theorem:Relation c and l}, we can prove Theorem
\ref{Theorem:lower bound for deep nets} as follows.

\begin{proof}[Proof of Theorem \ref{Theorem:lower bound for deep
nets}]
 It suffices to prove Theorem \ref{Theorem:lower bound for deep nets} for $p=1$,
since $2^d\|f\|_{L_p(\mathbb I^d)}\geq \|f\|_{L_1(\mathbb I^d)}$ for
arbitrary $f\in L_p(\mathbb I^d)$ and $p\geq 1$.
 Due to Theorem \ref{Theorem:covering number},
  (\ref{covering condition}) in Theorem
\ref{theorem:Relation c and l} is satisfied with $V=\mathcal
H_{n,L,\gamma,\mathcal
     R}$, $\tilde{C}_1=1$,
$\beta=0$ and $\tilde{C}_2= \left(c_3\mathcal
RD_{\max}\right)^{2(L+1)L}$. Hence, it follows from Theorem
\ref{theorem:Relation c and l} that
\begin{eqnarray*}
   &&\mbox{dist}(Lip^{(r,c_0)},\mathcal
H_{n,L,\sigma,\mathcal
     R\}},L_1(\mathbb I^d))\\
     &\geq&
   C'
   \left[2\mathcal A_L\log_2(2\mathcal A_L+1)\right]^{-\frac{r}{d}}.
\end{eqnarray*}
where
\begin{eqnarray*}
     &&C'=
     \frac14 d^{-d/2}\left[32(1+3r/d)
    \left(\log_2(2+8d^{d/2}(1+3r/d \right.\right.\\
     &+&\left.\left.        \left(c_3\mathcal
RD_{\max}\right)^{2(L+1)L})+1\right)\right]^{-\frac{r}{d}}.
\end{eqnarray*}
Since
\begin{eqnarray*}
       && 2+8d^{d/2}(1+3r/d
      +    \left(c_3\mathcal
      RD_{\max}\right)^{2(L+1)L})\\
      &\leq&
 (48d^{d/2}c_3\mathcal R D_{\max})^{2(L+1)L},
\end{eqnarray*}
and
$$
    \log_2(48d^{d/2}c_3\mathcal R
       D_{\max})\leq
       (\log_2(48d^{d/2}c_3)+1)\log_2(\mathcal R
       D_{\max}),
$$
we have
\begin{eqnarray*}
       C'
       &\geq&
       \bar{C}_1' [L^2\log_2(\mathcal R
       D_{\max})]^{-r/d}
\end{eqnarray*}
 where
  $\bar{C}_1':=\frac12\left[128(1+3r/d)(\log_2(48d^{d/2}c_3+1))\right]^{-\frac{r}{d}}$.
 Therefore,
\begin{eqnarray*}
   &&\mbox{dist}(Lip^{(r,c_0)},\mathcal
     H_{n,L,\sigma,\mathcal
     R},L_1(\mathbb I^d))\\
     &\geq&
      \bar{C}_1' [L^2\log_2(\mathcal R
       D_{\max})]^{-\frac{r}{d}}
   \left[2n\log_2(2n+1)\right]^{-\frac{r}{d}}\\
   &\geq&
   C[L^2n\log_2n \log_2(\mathcal R
       D_{\max})]^{-\frac{r}d}
\end{eqnarray*}
with $C=3^{-r/d}\bar{C}_1'$.
 This
completes the proof of Theorem \ref{Theorem:lower bound for deep
nets}.
\end{proof}

\section*{Acknowledgement}
The research was supported by the National Natural Science
Foundation of China [Grant Nos. 11531013, 61816133, 11571078,
11631015]. Lei Shi is also supported by the Program of Shanghai
Subject Chief Scientist (Project No.18XD1400700).


\begin{thebibliography}{99}

\bibitem{Adeli2009}
 H. Adeli and A. Panakkat. A probabilistic neural network for earthquake
magnitude prediction. Neural Networks,  vol. 22, pp. 1018-1024, 2009


\bibitem{Bianchini2014}
M. Bianchini and F. Scarselli. ``On the complexity of neural network
classifiers: a comparison between shallow and deep architectures'',
IEEE. Trans. Neural Netw. \& Learn. Sys., vol. 25, pp. 1553-1565,
2014.


 \bibitem{Chui1994}
C. K. Chui, X. Li and H. N. Mhaskar. Neural networks for lozalized
approximation. Math. Comput., vol. 63, pp. 607-623, 1994.

\bibitem{Chui1996}
C. K. Chui, X. Li and H. N. Mhaskar. Limitations of the
approximation capabilities of neural networks with one hidden layer.
Adv. Comput. Math., vol. 5, pp. 233-243, 1996.

\bibitem{Chui2018}
C. K. Chui, S. B. Lin and D. X. Zhou. Construction of neural
networks for realization of localized deep
  learning.  Front. Appl. Math. Statis., vol. 4, no. 14, 2018.

\bibitem{Chui2018a}
C. K. Chui, S. B. Lin and D. X. Zhou, Deep Neural Networks for Rotation-Invariance Approximation and Learning. Submitted to Analysis and Applications.




\bibitem{Delalleau2011}
O. Delalleau and Y. Bengio. ``Shallow vs. deep sum-product
networks'', NIPs, 666-674, 2011.



\bibitem{Goodfellow2016}
I. Goodfellow, Y. Bengio and A. Courville. {Deep Learning}. MIT
Press, 2016.


\bibitem{Gyorfi2002}
L. Gy\"{o}rfy, M. Kohler, A. Krzyzak and H. Walk. A
Distribution-Free Theory of Nonparametric Regression. Springer,
Berlin, 2002.


\bibitem{Hanin2017}
B. Hanin. Universal function approximation by deep neural nets with
bounded width and relu activations. arXiv preprint arXiv:1708.02691,
2017.

\bibitem{Harvey2017}
N. Harvey, C. Liaw and A. Mehrabian. Nearly-tight VC-dimension
bounds for piecewise linear neural networks. Conference on Learning
Theory. 2017: 1064-1068.


\bibitem{Hinton2006}
G. E. Hinton, S. Osindero and Y. W. Teh. A fast learning algorithm
for deep belief networks. Neural Comput., vol. 18, pp. 1527-1554,
2006.


\bibitem{Ismailov2014}
V. E. Ismailov.  On the approximation by neural networks with
bounded number of neurons in hidden layers.  J. Math. Anal. Appl.,
vol. 417, pp. 963-969, 2014.


\bibitem{Kohler2017}
M. Kohler and A. Krzyzak. Nonparametric regression based on
hierarchical interaction models.  IEEE Trans. Inf. Theory, vol. 63,
pp. 1620-1630, 2017.


\bibitem{Konovalov2008}
V. N. Konovalov, D. Leviatan and V. E. Maiorov.  Approximation by
polynomials and ridge functions of classes of s-monotone radial
functions. J. Approx. Theory,  vol. 152, pp. 20-51, 2008.

\bibitem{Konovalov2009}
V. N. Konovalov, D.  Leviatan and V. E. Maiorov. Approximation of
Sobolev classes by polynomials and ridge functions. J. Approx.
Theory,  vol. 159, pp. 97-108, 2009.


\bibitem{Kurkova2007}
V. K\r{u}rkov\'{a} and M. Sanguineti.   Estimates of covering
numbers of convex sets with slowly decaying orthogonal subsets.
Discrete Appl. Math., vol. 155, pp. 1930-1942, 2007.



\bibitem{Kurkova2017}
V. K\r{u}rkov\'{a} and M. Sanguineti. Probabilistic lower bounds for
approximation by shallow perceptron networks. Neural Networks, vol.
91, pp. 34-41, 2017.

\bibitem{Kurkova2018}
V. K\r{u}rkov\'{a}.  Constructive lower bounds on model complexity
of shallow perceptron networks.  Neural Comput. Appl., Doi:
10.1007/s00521-017-2965-0.

\bibitem{Krizhevsky2012}
A. Krizhevsky, I. Sutskever and G. E. Hinton. Imagenet
classification with deep convolutional neural networks. NIPS,
2097-1105, 2012.



\bibitem{Lecun2015}
Y. LeCun, Y. Bengio and G. Hinton. Deep learning.  Nature, vol. 521,
no. 7553, pp. 436-444, 2015.

\bibitem{LinH2017}
H. W. Lin, M. Tegmark and D. Rolnick. Why does deep and cheap
learning works so well? J. Stat. Phys., vol. 168, pp. 1223-1247,
2017.

\bibitem{Lin2017}
S. B. Lin, J.Zeng and X. Chang, Learning rates for classification
with Gaussian kernels. Neural Comput., vol. 29, pp. 3353-3380, 2017.

\bibitem{Lin2017a}
S. B. Lin. Limitations of shallow nets approximation. Neural
Networks,   vol. 94, pp. 96-102, 2017.


\bibitem{Lin2018}
S. B. Lin. Generalization and expressivity for deep nets. IEEE
Trans. Neural Netw. Learn. Syst., In Press.


\bibitem{LinT2008}
T. Lin and H. Zha. Riemannian manifold learning. IEEE Trans. Pattern
Anal. Mach. Intel., vol. 30, pp. 796-809, 2008.

\bibitem{Lorentz1996}
G. G. Lorentz, M. Z. Golitchek and Y. Makovoz. Constructive
Approximation, Advanced Problems.  Springer-Verlag, New York, 1996.


\bibitem{Maiorov1999b}
V. Maiorov and A. Pinkus. Lower bounds for approximation by MLP
neural networks. Neurocomputing,  vol. 25, pp. 81-91, 1999.

\bibitem{Maiorov1999c}
V. Maiorov and J. Ratsaby. On the degree of approximation by
manifolds of finite pseudo-dimension. Constr. Approx., vol. 15, pp.
291-300, 1999.

\bibitem{Maiorov2000}
V. Maiorov and R. Meir. On the near optimality of the stochastic
approximation of smooth functions by neural networks. Adv. Comput.
Math.,  vol. 13, pp. 79-103, 2000.

\bibitem{Maiorov2006}
V.  Maiorov. Pseudo-dimension and entropy of manifolds formed by
affine invariant dictionary.  Adv. Comput. Math., vol. 25, pp.
435-450, 2006.


\bibitem{Makovoz1996}
Y. Makovoz.  Random approximants and neural networks. J. Approx.
Theory, vol. 85, pp. 98-109, 1996.


\bibitem{McCane2017}
B. McCane and L. Szymanski. Deep radial kernel networks:
approximating radially symmetric functions with deep networks. arXiv
preprint arXiv:1703.03470, 2017.

\bibitem{Meylan2006}
L. Meylan and S. Susstrunk. High dynamic range image rendering with
a retinex-based adaptive filter. IEEE Trans. Image Proc.,  vol. 15,
pp.   2820-2830, 2006.

%
%
\bibitem{Mhaskar2016}
H.   Mhaskar, Q. Liao and T. Poggio. Learning real and Boolean
functions: When is deep better than shallow. arXiv preprint
arXiv:1603.00988, 2016.

\bibitem{Mhaskar2016a}
H. N. Mhaskar and T. Poggio. Deep vs. shallow networks: An
approximation theory perspective, Anal. Appl., vol. 14, pp. 829-848,
2016.


\bibitem{Montufar2013}
G. Mont\'{u}far, R. pascanu, K. Cho and Y. Bengio.  On the number of
linear regions of deep nerual networks. Nips, 2014,  2924-2932.



\bibitem{Petersen2017}
P. Petersen and F. Voigtlaender. Optimal aproximation of piecewise
smooth functions using deep ReLU neural networks, Neural Networks,
vol. 108, pp. 296-330, 2018.



\bibitem{Poggio2017}
T. Poggio, H. Mhaskar, L. Rosasco, B. Miranda and Q. Liao. Why and
when can deep-but not shallow-networks avoid the curse of
dimensionality: A review. Intern. J.   Auto. Comput.,  DOI:
10.1007/s11633-017-1054-2, 2017.
%


\bibitem{Rolnick2017}
D. Rolnick and M. Tegmark. The power of deeper networks for
expressing natural functions. arXiv:1705.05502v1, 2017.


\bibitem{Safran2016}
I. Safran and O. Shamir. Dept-width tradeoffs in approximating
natural functions with neural networks. arXiv reprint
arXiv:1610.09887v2, 2016.

\bibitem{Satriano2011}
C. Satriano, Y. M. Wu, A. Zollo and H. Kanamori. Earthquake early
warning: Concepts, methods and physical grounds. Soil Dynamics
Earth. Engineer., vol. 31, pp. 106-118, 2011.



\bibitem{Shi2011}
L. Shi, Y. L. Feng and  D. X. Zhou.   Concentration estimates for
learning with $l_1$-regularizer and data dependent hypothesis
spaces. Appl. Comput. Harmon. Anal.,   vol. 31, pp.  286-302, 2011.

\bibitem{Shi2013}
L. Shi.  Learning theory estimates for coefficient-based regularized
regression.  Appl. Comput. Harmon. Anal., vol. 34, pp. 252-265,
2013.



\bibitem{Shaham2015}
U. Shaham,  A. Cloninger and R. R. Coifman, Provable approximation
properties for deep neural networks. Appl. Comput. Harmon. Anal., To
appear.


\bibitem{Vikraman2016}
K. Vikraman. A deep neural network to identify foreshocks in real
time, arXiv preprint arXiv:1611.08655, 2016.

\bibitem{Wang2004}
Z. Wang, A. C. Bovik, H. R. Sheikh and E. P. Simoncelli. Image
quality assessment: from error visibility to structural similarity.
IEEE Trans. Image Process., vol. 13, pp. 600-612, 2004.


\bibitem{Wright2010}
J. Wright, Y. Ma, J. Mairal, G. Sapiro, T. S. Huang and S.  Yan.
Sparse representation for computer vision and pattern recognition.
Proc. IEEE, vol. 98, pp. 1031-1044, 2010.

\bibitem{Yarotsky2017}
D. Yarotsky. Error bounds for aproximations with deep ReLU networks.
Neural Networks, vol. 94, pp. 103-114, 2017.

\bibitem{Ye2008}
G. B. Ye and D. X. Zhou. Learning and approximation by Gaussians on
Riemannian manifolds.  Adv. Comput. Math., vol. 29, pp. 291-310,
2008.

\bibitem{Zhou2002}
D. X. Zhou. The covering number in learning theory. J. Complex.,
  vol. 18, pp. 739-767, 2002.

\bibitem{Zhou2003}
 D. X. Zhou. Capacity of reproducing kernel spaces in learning
 theory.
IEEE Trans. Inf. Theory, vol. 49, pp. 1743-1752, 2003.

\bibitem{Zhou2018}
D. X. Zhou. Deep distributed convolutional neural networks:
Universality. Anal. Appl., In Press, 2018.

\bibitem{Zhou2018a}
D. X. Zhou. Universality of Deep Convolutional Neural Networks.
 arXiv:1805.10769, 2018.
%
%
%
%
%
%
%
%
%
%
%
%
%
%
%
%
%
%
%
%
%
%
%
%
%
%
%
%
%
%

%
%
%

%
%
%
%
%
%
%
%
%
%
%
%
%
%

%
%
%
%
%
%
%
%
%
%
%
%
%
%
%
%
%
%
%
%
%
%
%
%
%
%
%
%
%
%
%
%
%
%
%
%
%
%

%
%

%

\end{thebibliography}
\end{document}